\documentclass{article}

    \PassOptionsToPackage{numbers, compress}{natbib}
 \usepackage[preprint]{neurips_2026}

\newif\ifshowchecklist
\makeatletter
\@ifpackagewith{neurips_2026}{preprint}{\showchecklistfalse}{\showchecklisttrue}
\makeatother

\usepackage[utf8]{inputenc} 
\usepackage[T1]{fontenc}    
\usepackage{hyperref}       
\usepackage{url}            
\usepackage{booktabs}       
\usepackage{amsfonts}       
\usepackage{nicefrac}       
\usepackage{microtype}      
\usepackage{xcolor}         

\newcommand{\atopn}{Atop-N}

\usepackage{amsmath}
\usepackage{amssymb}
\usepackage{mathtools}
\usepackage{amsthm}
\theoremstyle{plain}
\newtheorem{theorem}{Theorem}[section]

\newtheorem{lemma}[theorem]{Lemma}

\theoremstyle{definition}

\theoremstyle{remark}

\usepackage{amsmath,amsfonts,bm}
\usepackage{nicefrac}

\def\eqref#1{equation~\ref{#1}}

\def\1{\bm{1}}

\def\rvi{{\mathbf{i}}}

\def\rmA{{\mathbf{A}}}
\def\rmB{{\mathbf{B}}}
\def\rmC{{\mathbf{C}}}

\def\rmK{{\mathbf{K}}}

\def\rmM{{\mathbf{M}}}

\def\rmO{{\mathbf{O}}}
\def\rmP{{\mathbf{P}}}
\def\rmQ{{\mathbf{Q}}}

\def\rmS{{\mathbf{S}}}

\def\rmV{{\mathbf{V}}}
\def\rmW{{\mathbf{W}}}
\def\rmX{{\mathbf{X}}}

\DeclareMathAlphabet{\mathsfit}{\encodingdefault}{\sfdefault}{m}{sl}
\SetMathAlphabet{\mathsfit}{bold}{\encodingdefault}{\sfdefault}{bx}{n}

\newcommand{\R}{\mathbb{R}}

\newcommand{\softmax}{\mathrm{softmax}}

\usepackage{xspace}
\makeatletter
\DeclareRobustCommand\onedot{\futurelet\@let@token\@onedot}
\def\@onedot{\ifx\@let@token.\else.\null\fi\xspace}

\def\etal{\emph{et al}\onedot}
\makeatother

\newcommand{\email}[1]{\scalebox{0.88}[1]{\texttt{#1}}}

\usepackage[capitalize,noabbrev]{cleveref}
\crefname{section}{Sec.}{Secs.}
\Crefname{section}{Section}{Sections}
\crefname{equation}{Eq.}{Eqs.}
\Crefname{equation}{Equation}{Equations}
\crefname{figure}{Fig.}{Figs.}
\Crefname{figure}{Figure}{Figures}
\crefname{table}{Tab.}{Tabs.}
\Crefname{table}{Table}{Tables}
\crefname{appendix}{App.}{Apps.}
\Crefname{appendix}{Appendix}{Appendices}

\title{More Value per Key: Asymmetric Sparse Attention for Faster LLM Decoding}

\author{%
  Noam Elata\thanks{Equal contribution.} \\
  Technion -- Haifa, Israel \\
  \email{noamelata@technion.ac.il} \\
  \And
  Itay Lamprecht\footnotemark[1]\hspace{0.5em}\thanks{Work was partially done while the author was at Intel, Israel.} \\
  Technion -- Haifa, Israel \\
  Crusoe AI \\
  \makebox[0pt][c]{\email{itayl@campus.technion.ac.il}}
  \And
  Mikey Shechter\footnotemark[1] \\
  Technion -- Haifa, Israel \\
  \email{m.sachter@technion.ac.il} \\
  \And
  Daniel Ohayon \\
  Corma \\
  Technion -- Haifa, Israel \\
  \email{ohayon.daniel4@gmail.com} \\
  \And
  Itay Hubara \\
  Stealth Startup \\
  \email{itayhubara@gmail.com} \\
  \And
  Daniel Soudry \\
  Technion -- Haifa, Israel \\
  \email{daniel.soudry@gmail.com} \\
}

\begin{document}

\maketitle
\setcounter{footnote}{0}

\begin{abstract}
    Autoregressive generation in Large Language Models (LLMs) is constrained by the memory and computational demands of attention mechanisms. Sparse attention methods mitigate this cost by selecting only high-probability entries of the attention matrix. We observe that in many such methods, this renders the probability-value multiplication negligible, shifting the bottleneck to the query-key step. Key heads can therefore be reduced to accelerate inference, while retaining more value heads preserves capacity with limited additional decoding cost. We introduce Sparse Asymmetric Group-Query Attention (SAGA), which decouples key and value head counts to exploit this principle, and pair it with approximate top-N (Atop-N) attention, a simple sparse attention method designed to study the interaction between sparsity and head-count asymmetry. We formalize the benefits of this asymmetry theoretically and validate them empirically through latency measurements and quality evaluations on models up to 1.5B parameters. Together, SAGA and Atop-N achieve end-to-end decoding speedups exceeding $2\times$ over our full-attention GQA baseline at long contexts. Models trained from scratch with SAGA nearly match the quality of comparable GQA variants on the evaluated benchmarks. To facilitate adoption, we introduce an efficient fine-tuning method that converts pretrained models to the SAGA architecture, enabling practitioners to benefit from our approach without costly retraining.\footnote{Code is available at \url{https://github.com/noamelata/SAGA}.}
\end{abstract}

\section{Introduction}

Attention-based architectures have become dominant across AI applications~\cite{vaswani2017attention}, with large language models such as GPT~\cite{radford2019language, brown2020language}, Llama~\cite{touvron2023llama}, and Qwen~\cite{qwen2, qwen2.5} demonstrating strong capabilities in natural language understanding~\cite{wang2019superglue}, generation~\cite{papineni2002bleu}, and reasoning~\cite{wei2022chain}. However, the quadratic scaling of attention with sequence length~\cite{shen2021efficient, keles2023computational} remains a significant bottleneck, constraining context window sizes and limiting the ability of models to reason over lengthy documents.

For autoregressive generation, KV caching~\cite{pope2023efficiently} provides substantial speedups by reusing key and value computations across generation steps. This shifts the bottleneck from compute capacity to memory capacity and bandwidth~\cite{davies2025liminalexploringfrontiersllm}. Group-Query Attention (GQA)~\cite{ainslie2023gqa} further exploits this setting by reducing the number of key-value heads, accelerating inference for long sequences. Yet even with these optimizations, bounded context windows remain a practical necessity. 

Sparse attention methods~\cite{child2019generating, beltagy2020longformer, zaheer2020big, kitaev2020reformerefficienttransformer, yuan2025nsa} offer a complementary direction, accelerating attention by computing only a subset of query-key interactions. Such methods span both the prefill and decoding phases, and range from zero-shot approximations to approaches integrated into pre-training. In this work, we focus on the decoding phase and use a generic sparsity method, as primary examples of the effects of sparsity on the computational profile of generation. To this end, we adopt top-N attention as a representative sparsity method. Building on previous observations~\cite{gupta2021memory, wang2022kvt} that models assign significant attention weight to only a small subset of keys, top-N attention has the advantage of simplicity: it requires no training to apply and admits a portable implementation, a property uncommon among sparse attention methods. Our efficient approximate top-N attention (\atopn) implementation gathers only the most relevant attention entries, reducing the probability-value multiplication to a fraction of its original cost.

Analyzing the resulting computational profile reveals a fundamental asymmetry: under sparse attention, the bottleneck shifts from the probability-value multiplication to the query-key multiplication. This insight motivates Sparse Asymmetric Group-Query Attention (SAGA), an architectural modification that decouples the number of key heads $h_k$ from value heads $h_v$. While GQA ties these together, SAGA reduces $h_k$ independently, directly accelerating the query-key multiplication. The number of value heads remains large to preserve model capacity, incurring no additional cost since only the top-N values are accessed. We formalize this asymmetry theoretically and validate it empirically, demonstrating end-to-end speedups exceeding $2\times$ in long contexts while maintaining the majority of baseline quality.

To facilitate practical adoption, we introduce an efficient fine-tuning method for converting existing GQA models to the SAGA architecture. We merge teacher key heads using a closed-form initialization derived from minimum-weight perfect matching on attention logit reconstruction error (\cref{sec:gqa_to_saga}), followed by a short distillation phase. This enables the benefits of SAGA without costly retraining.

To summarize, our contributions are:
\vspace{-0.5em}
\begin{itemize}
    \item We present a novel analysis of LLM inference trade-offs with accelerated sparse attention.
    \item We introduce \atopn, a simple and portable sparse attention algorithm for decoding.
    \item Motivated by our finding, we propose SAGA, an extension of the GQA architecture tailored for sparse attention inference. We show that SAGA is competitive both when trained from scratch and when finetuned from existing LLM transformers, while demonstrating more than $2\times$ runtime speedups end-to-end on long sequence generation.
\end{itemize}

\section{Background and related work}

\subsection{Multi-head self-attention}

The Transformer architecture, introduced by Vaswani \etal~\cite{vaswani2017attention}, revolutionized natural language processing. Its core mechanism, multi-head self-attention, enables parallel processing of sequences and captures long-range dependencies by allowing the model to attend to all positions simultaneously.

Formally, given a hidden state $\rmX \in \R^{l_q \times D}$, we apply three projections $\rmW_Q, \rmW_K, \rmW_V \in \R^{D \times (h \cdot d)}$ to obtain queries $\rmQ \in \R^{h \times l_q \times d}$, keys $\rmK \in \R^{h \times l_{kv} \times d}$, and values $\rmV \in \R^{h \times l_{kv} \times d}$:
\begin{equation}
\label{eq:projections}
    \rmQ = \rmX \rmW_Q, \quad \rmK = \rmX \rmW_K, \quad \rmV = \rmX \rmW_V, 
\end{equation}
where $h$ is the number of heads, $D$ and $d$ are the hidden and head dimension sizes respectively, and $l_q$ and $l_{kv}$ are the query and key-value sequence lengths. The attention output is computed as:

\begin{minipage}[c]{0.32\columnwidth}
  \begin{equation}
  \label{eq:qk}
  \rmA = \frac{\rmQ \rmK^\top}{\sqrt{d}}
  \end{equation}
\end{minipage}
\hfill
\begin{minipage}[c]{0.32\columnwidth}
\setlength{\abovedisplayskip}{0pt}
\setlength{\belowdisplayskip}{0pt}
  \begin{equation}
  \label{eq:softmax}
      \rmP = \text{softmax}(\rmA)
  \end{equation}
\end{minipage}
\hfill
\begin{minipage}[c]{0.32\columnwidth}
\setlength{\abovedisplayskip}{0pt}
\setlength{\belowdisplayskip}{0pt}
  \begin{equation}
  \label{eq:pv}
      \rmO = \rmP \rmV,
  \end{equation}
\end{minipage}

\noindent where $\rmA, \rmP \in \R^{h \times l_q \times l_{kv}}$ denote the attention and probability matrix respectively. The output $\rmO$ is then multiplied by $\rmW_O\in \R^{(h \cdot d) \times D}$. Modern LLMs stack many such layers in a decoder-only design, modeling complex relationships between all sequence elements regardless of distance.

\subsection{KV cache}

Key-Value (KV) caching~\cite{pope2023efficiently} is an optimization technique that significantly improves the inference efficiency of autoregressive Transformer models. During the generation of each new token, Transformer models would typically recompute the key and value projections for all previous tokens in the sequence. The KV cache stores  projections from previous steps, eliminating redundant computations.

For an autoregressive model generating the $l_{kv}$-th token, the attention computation with KV caching can be expressed as:
\begin{equation}
    \rmO_{l_{kv}} = \text{softmax}\left(\frac{\rmQ_{l_{q}} [\rmK_{1:({l_{kv}}-1)}, \rmK_{l_{kv}}]^\top}{\sqrt{d}}\right)[\rmV_{1:({l_{kv}}-1)}, \rmV_{l_{kv}}],
\end{equation}
where $\rmK_{1:({l_{kv}}-1)}$ and $\rmV_{1:({l_{kv}}-1)}$ are retrieved from the cache, and only $\rmQ_{l_{q}}$, $\rmK_{l_{kv}}$, and $\rmV_{l_{kv}}$ need to be computed for the new token.

By caching key and value vectors from previous time steps, the model only needs to compute new projections for the most recently generated token, leading to substantial speedups during inference. This optimization is particularly valuable for long-form text generation, where the model must generate many tokens sequentially. The trade-off is increased memory usage proportional to sequence length, as the cache must store key-value pairs for all tokens in the context window.

\subsection{Group-query attention and head reduction}

Group-Query Attention (GQA)~\cite{ainslie2023gqa} addresses the memory overhead of KV caching by recognizing an asymmetry in the attention computation during decoding: there are significantly more cached keys and values than queries ($l_{kv} \gg l_q$). In this context, it is sensible to allocate more of the model's capacity to the queries, which are computed fresh at each decoding step, rather than to the keys and values, which are retrieved from the cache.

GQA achieves this by decoupling the number of key-value heads $h_{kv}$ from the number of query heads $h_q$, such that $h_q > h_{kv}$. Multiple query heads share a single key-value head, reducing the size of the KV cache while preserving the expressiveness of the query projections. This approach yields a significant improvement in inference efficiency with negligible degradation in model quality, making it a widely adopted technique in modern LLM architectures.

Beyond reducing the total number of KV heads, recent work has explored making key and value head counts asymmetric. DHA~\cite{chen2024dha} fuses similar attention heads in MHA architecture into decoupled per-layer key and value configurations to reduce memory footprint. Sigma~\cite{lin2025sigma} performs differential rescaling of query, key, and value dimensions during pretraining and finds that reducing key capacity is preferable to reducing value capacity---an observation consistent with our analysis. While conceptually related, SAGA differs in both motivation and applicability, exploiting the shifted computational bottleneck under sparse attention to accelerate decoding. Multi-head latent attention with DeepSeek Sparse Attention~\cite{liu2024deepseek, deepseekv32} is related in spirit but must be trained from scratch with a shared latent KV and a learned indexer; it reduces the number of key \emph{positions} scored per query, whereas SAGA reduces the number of key \emph{heads}, so the two are orthogonal rather than competing.

\begin{figure*}[t]
  \vskip 0.2in
  \begin{center}
    \centerline{\includegraphics[width=\textwidth]{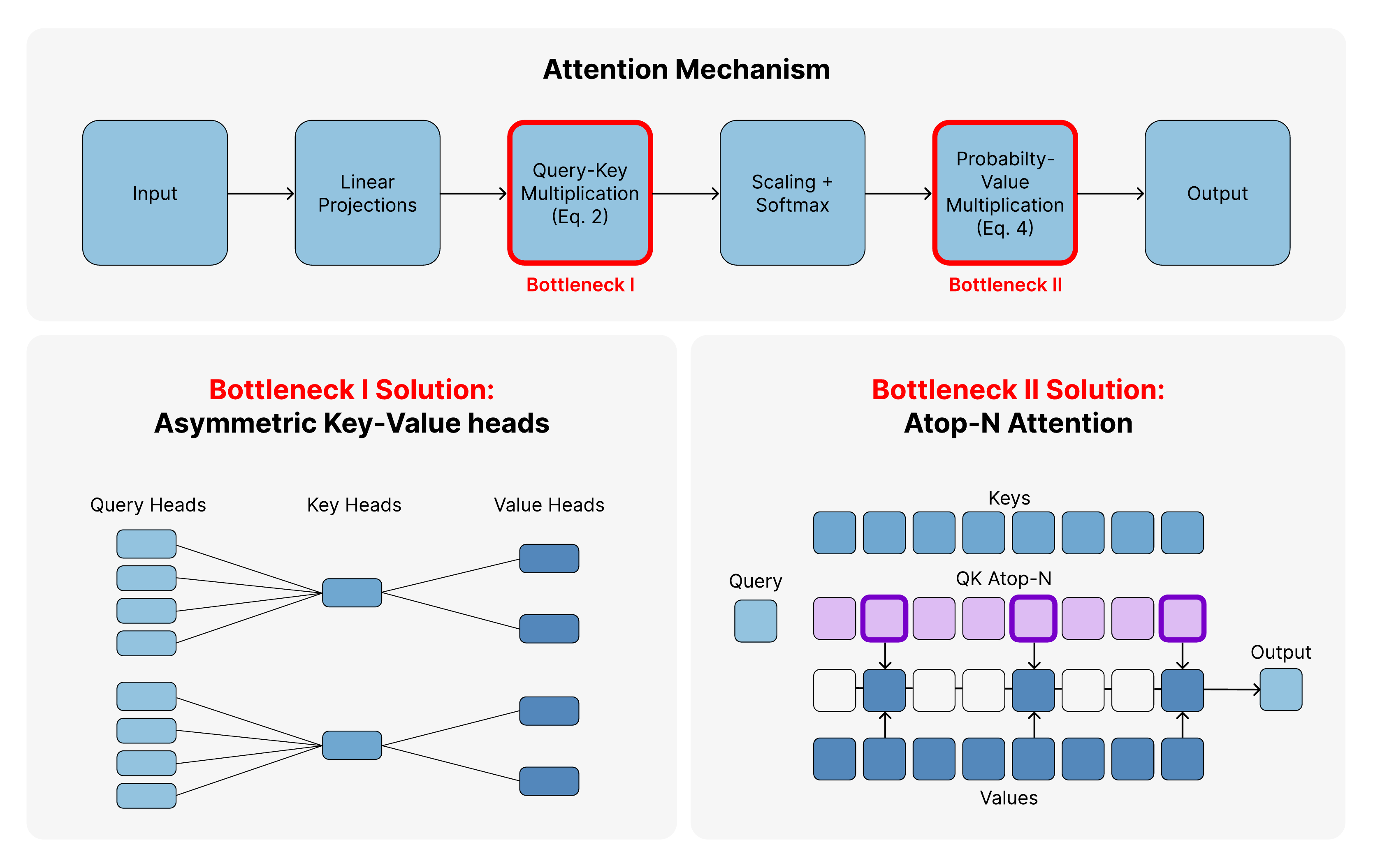}}
    \caption{
      Diagram of SAGA. There are two significant decoding bottlenecks in attention mechanisms: key-query multiplications and probability-value multiplications (top). SAGA addresses the query-key bottleneck by reducing the number of key heads (left). The probability-value bottleneck is mitigated by utilizing Atop-N to select important values (right).  
    }
    \label{fig:saga-diagram}
  \end{center}
\end{figure*}

\subsection{Attention sparsity and KV cache acceleration}

Recent efforts to mitigate the quadratic complexity of the attention mechanism have largely focused on attention sparsity and KV cache compression. Sparse attention methods, such as BigBird \cite{zaheer2020big}, Longformer \cite{beltagy2020longformer}, and Sparse Transformer \cite{child2019generating}, utilize predefined patterns like sliding windows or global tokens to restrict the number of query-key pairs computed. While effective for the prefill phase, these static patterns often struggle to capture the dynamic, data-dependent importance of tokens during autoregressive decoding. NSA~\cite{yuan2025nsa}, SpargeAttention~\cite{zhangspargeattention}, BLASST~\cite{yuan2025blasst}, and BSFA~\cite{ohayon2025block} also aim to bridge the gap between algorithmic sparsity and hardware efficiency, mostly focusing on the more compute-intensive pre-fill phase. Our work, SAGA, focuses instead on the decode phase, and offers a new framework by combining a simple attention sparsity method with a dedicated neural architecture.

Other methods accelerate decoding by evicting low-importance tokens from the KV cache~\cite{zhang2023h2o, xiao2024efficientstreaminglanguagemodels} or selectively loading cache pages using query-aware metadata~\cite{li2024snapkv, tang2024quest}. Complementary techniques reduce cache memory via quantization~\cite{liu2024minicache, sheng2023flexgen}. SAGA is orthogonal to these approaches; it optimizes the head architecture itself and can be readily combined with cache compression methods.

\subsection{Top-N attention}

Top-N attention~\cite{gupta2021memory, wu2022memorizing} (also known as top-K attention or $k$NN attention) is based on the observation that only a small subset of attention keys meaningfully influence the attention output. This sparsity arises from the softmax operation, which effectively learns to identify the best query-key matches while suppressing mismatched pairs.

Formally, top-N attention identifies the indices of the $N$ largest values in each row of $\rmA$:
\begin{equation}
\label{eq:top-n}
    \rvi_{j,k,1:N} = \text{TopN}(\rmA_{j,k,1:l_{kv}}),
\end{equation}
and masks all other entries of $\rmA$ not in $\rvi$ before applying the softmax. This formulation has been shown to enable memory-efficient implementations~\cite{gupta2021memory, wu2022memorizing}, though it does not yield computational speedups since the full attention matrix must still be computed. Several recent works have extended top-N attention in various directions~\cite{bertsch2023unlimiformer, lou2024sparser} or strengthened its theoretical foundations~\cite{haris2024k}.

\section{Method}

We begin by introducing the attention sparsity method underlying SAGA. For clarity, we present an architecture-agnostic formulation of top-N attention to illustrate our approach and the trade-offs many online sparse attention method induce. In the following analysis, we show that under sparse attention, the computational bottleneck shifts from the probability-value multiplication to the query-key multiplication---the property that SAGA is designed to exploit.

\subsection{Efficient approximate top-N attention for LLM decoding}
\label{sec:top-n-attn}

We propose an efficient framework for top-N attention designed for LLM decoding. Using the indices $\rvi \in \mathbb{Z}^{h \times l_q \times N}$ from \cref{eq:top-n}, we define the \textbf{Top-N Attention Matrix} $\rmA[\rvi] \in \R^{h \times l_q \times N}$ and corresponding \textbf{Top-N Values} $\rmV[\rvi]
\in \R^{h \times l_q \times N \times d}$, and compute attention over only the gathered entries:
\begin{equation}
  \label{eq:gathered-attn}
  \rmO = \softmax(\rmA[\rvi]) \rmV[\rvi ].
\end{equation}
In autoregressive generation ($l_q = 1$), this reduces the probability-value cost from $O(h\, l_{kv}\, d)$ to $O(h\, N\, d)$, potentially saving nearly 50\% of computation and bandwidth when $N \ll l_{kv}$. However, na\"ively finding the true top-N entries incurs lengthy hardware runtimes, negating the benefit.

We address this with an approximate threshold approach: given the $N$-th largest attention value as a threshold, top-N entries can be identified by a simple per-element comparison --- a parallel, portable operation much faster than sorting. Since attention distributions shift only marginally between consecutive decoding steps, we compute exact thresholds once every $n$ tokens and reuse them, amortizing the cost over the full text generation. We term this method \atopn{} attention; full implementation details are in \cref{app:implementation}.

For long sequences, \atopn{} yields an overall speedup of nearly $2\times$ for the full attention operation. \Cref{sec:experiments} shows that for practical values of $N$, \atopn{} retains 83--97\% of baseline RULER accuracy.

\subsection{Sparse Asymmetric Group-Query Attention}

The analysis in the previous section reveals that under \atopn{} attention, the probability-value multiplication accesses only $h_q \times l_q \times N$ value vectors, $N$ per query head, regardless of the number of key-value heads $h_{kv}$. Consequently, $h_{kv}$ affects computational complexity only through the query-key multiplication in \cref{eq:qk}. This asymmetry suggests that the number of key heads and value heads should be disentangled.

We propose Sparse Asymmetric Group-Query Attention (SAGA), which decouples the key and value head counts into separate hyperparameters $h_k$ and $h_v$. Reducing $h_k$ directly reduces the memory bandwidth required to load keys during the query-key multiplication step (\cref{eq:qk}), the dominant latency bottleneck in LLM decoding under \atopn{} attention. The number of value heads $h_v$ can remain large to preserve model capacity, since only the top-N values are accessed regardless of $h_v$. A diagram of SAGA is shown in ~\cref{fig:saga-diagram}.

Beyond the bandwidth reduction, fewer key heads yield a smaller key projection $\rmW_K$ (reducing per-token compute) and a smaller KV cache footprint, enabling longer context lengths or larger batch sizes within the same memory budget.

SAGA can be applied directly when training from scratch, as detailed in \cref{sec:scratch}. To leverage existing pretrained models, we describe a procedure for converting GQA models to SAGA in \cref{sec:gqa_to_saga} below.

\section{Converting GQA models to SAGA}
\label{sec:gqa_to_saga}
We describe a practical procedure for converting a pretrained GQA model into a SAGA model. Our goal is to transform the attention module of an existing GQA model $T$, with equal numbers of key and value heads
\begin{equation}
h_k^T = h_v^T,    
\end{equation}
into an attention module of a SAGA model $S$ with fewer key heads than value heads,
\begin{equation}
h_k^S < h_v^S,
\end{equation}
while preserving the attention behavior as closely as possible. The conversion consists of two stages:
\begin{enumerate}
    \item Choosing an initialization for the SAGA key projection matrices $\rmW_K^S$ that minimizes deviation from the original GQA attention, as described below.
    \item Distillation-based fine-tuning, using the original GQA model $T$ as a teacher and the SAGA model $S$ as a student, as implemented in \cref{sec:distillation}.
\end{enumerate} 
Our GQA to SAGA conversion recipe naturally extends to multi-head attention (MHA) models.

\subsection{Initializing SAGA from GQA}\label{sec:init_saga} We consider a GQA model with $h_k^T$ key heads, $h_v^T=h_k^T$ value heads, and $h_q$ query heads. We convert it to a SAGA model with fewer key heads $h_k^S$ while keeping the same number of value heads $h_v^S=h_v^T$. We assume the reduction factor is a positive integer $m \geq 2$
\begin{equation}
h_k^T = m \cdot h_k^S \qquad (m\in\mathbb{N}),
\end{equation}
so that each SAGA key head is shared by exactly $m$ value heads. For simplicity we describe the case $m=2$, i.e.\ $h_k^T=2h_k^S$, and note that the construction extends directly to any $m \in \mathbb{N}$ by replacing the minimum-weight perfect matching in \Cref{thm:merge_theorem} with a minimum-weight perfect $m$-partitioning of $G$. Let $\rmW_K^1, \ldots, \rmW_K^{2h_k^S}$ denote the key projection matrices, and $\rmW_Q^1, \ldots, \rmW_Q^{h_q}$ the query projection matrices of an attention module of the teacher model $T$. 

Let $\rmX \in \mathbb{R}^{L \times D}$ denote a matrix of hidden states sampled from a representative dataset. Our objective is to construct new key projections $\hat{\rmW}_K^1, \ldots, \hat{\rmW}_K^{h_k^S}$ for the SAGA model such that the resulting attention logits
\begin{equation}
\frac{(\rmX\rmW_Q^j)(\rmX\hat{\rmW}_K^i)^\top}{\sqrt d}
\end{equation}
closely approximate the corresponding logits produced by the original GQA model for every query head $j$ that originally shared one of the two merged teacher key heads. 
We choose to optimize only the merged key projections, and to match the attention logits while ignoring rotary embeddings because this is the broadest objective that enables us to derive a closed-form solution. Matching attention probabilities, or jointly optimizing the query and key projections, leads to a nonlinear problem with no closed-form minimizer.

Reducing the number of key heads requires deciding which original key heads should be merged together. We formalize this choice using a permutation $\sigma = (\sigma_1, \ldots, \sigma_{2h_k^S})$ of $\{1,\ldots,2h_k^S\}$. The permutation specifies a partition of the original key heads into $h_k^S$ disjoint pairs:
\begin{equation}
(\sigma_{2i-1}, \sigma_{2i}), \quad i = 1,\ldots,h_k^S.
\end{equation}
Each pair is merged into a single SAGA key head. Once $\sigma$ is determined, we apply the induced permutations to the query, value, and output projection matrices so that they remain aligned with the reordered key heads and each merged key is shared by its intended value-head pair. The induced permutations definitions are given in \cref{app:permute_qvo}. Now we explain how we choose the permutation and merge the chosen key heads. For each teacher key head $r \in \{1,\ldots,2h_k^S\}$, let $\mathcal{Q}_r$ denote the query heads that share $\rmW_K^r$, and let $\rmQ^j := \rmX\rmW_Q^j$ denote the realized query matrices. For any pair of teacher key heads $a,b \in \{1,\ldots,2h_k^S\}$ and candidate merged projection $\rmM$, define the pairwise logit-matching loss
\begin{equation}
\begin{aligned}
\mathcal{L}_{a,b}(\rmM)
:=
\sum_{j \in \mathcal{Q}_a}
&
\left\lVert
\rmQ^j(\rmX\rmM)^\top
-
\rmQ^j(\rmX\rmW_K^a)^\top
\right\rVert_{\mathrm{F}}^2
\\
+
\sum_{j \in \mathcal{Q}_b}
&
\left\lVert
\rmQ^j(\rmX\rmM)^\top
-
\rmQ^j(\rmX\rmW_K^b)^\top
\right\rVert_{\mathrm{F}}^2
\end{aligned}
\end{equation}
where $\lVert A \rVert_{\mathrm{F}} = \sqrt{\sum_{i,j} A_{ij}^2}$ denotes the Frobenius norm of a matrix $A$.
We omit the attention scaling factor $1/\sqrt d$, since it multiplies the entire objective by the constant $1/d$ and therefore does not affect the minimizer. Our objective is to jointly choose the pairing $\sigma$ and the merged key projections $\rmM^{[1:h_k^S]}$ so as to minimize the total squared logit reconstruction error over the dataset:
\begin{equation}
\label{eq:merge_objective}
\begin{aligned}
\min_{\sigma}
\min_{\rmM^{[1:h_k^S]}}
\sum_{i=1}^{h_k^S}
\mathcal{L}_{\sigma_{2i-1},\sigma_{2i}}(\rmM^i).
\end{aligned}
\end{equation}

\begin{lemma}
\label{lem:pair_logit_merge}
For any pair of teacher key heads $a,b \in \{1,\ldots,2h_k^S\}$, an optimal merged key projection minimizing $\mathcal{L}_{a,b}(\rmM)$ is
\begin{equation}
\hat{\rmW}_K(a,b)
=
\left(
\rmW_K^a\rmS_a
+
\rmW_K^b\rmS_b
\right)
\left(
\rmS_a+\rmS_b
\right)^{\dagger}.
\end{equation}
where $\rmS_r := \sum_{j \in \mathcal{Q}_r} (\rmQ^j)^\top(\rmQ^j)$, and $\dagger$ denotes the Moore--Penrose pseudoinverse.
\end{lemma}
The proof of Lemma~\ref{lem:pair_logit_merge} is given in \cref{app:logit_merge}. 
Let $G = (V, E, w)$ be the complete weighted graph with vertex set $V = \{1, \ldots, 2h_k^S\}$, where for each pair of teacher key heads $\{a,b\}$, $E$ contains an edge with weight
\begin{equation}
w(a,b) := \mathcal{L}_{a,b}(\hat{\rmW}_K(a,b))
=
\min_{\rmM} \mathcal{L}_{a,b}(\rmM).
\end{equation}
where the second equality follows from Lemma~\ref{lem:pair_logit_merge}.

\begin{theorem}
\label{thm:merge_theorem}
A minimum-weight perfect matching of $G$ yields an optimal solution to \Cref{eq:merge_objective}, with $\sigma^*$ being a permutation whose consecutive pairs are the matched edges, and for each matched pair $(\sigma_{2i-1}^{*}, \sigma_{2i}^{*})$ the corresponding optimal merged key projection is
\begin{equation}
\hat{\rmW}_K^i
=
\hat{\rmW}_K(\sigma_{2i-1}^*, \sigma_{2i}^*).
\end{equation}
\end{theorem}

\begin{proof}
For a fixed permutation $\sigma$, we look at the inner objective in \cref{eq:merge_objective}. This objective is a sum of $h_k^S$ independent pairwise terms, one for each pair $(\sigma_{2i-1},\sigma_{2i})$. Therefore
\begin{equation}
\min_{\rmM^{[1:h_k^S]}}
\sum_{i=1}^{h_k^S}
\mathcal{L}_{\sigma_{2i-1},\sigma_{2i}}(\rmM^i)
=
\sum_{i=1}^{h_k^S}
\min_{\rmM^i} \mathcal{L}_{\sigma_{2i-1},\sigma_{2i}}(\rmM^i)
=
\sum_{i=1}^{h_k^S}
w(\sigma_{2i-1},\sigma_{2i}),
\end{equation}
Moreover, for each $i$, the minimizing choice is $\hat{\rmW}_K^i = \hat{\rmW}_K(\sigma_{2i-1},\sigma_{2i})$. Thus the value of \cref{eq:merge_objective} for a given permutation $\sigma$ is exactly the total weight of the perfect matching encoded by its consecutive pairs. Any perfect matching of $G$ can be written as such a permutation, and conversely any such permutation defines a perfect matching. Minimizing \cref{eq:merge_objective} over $\sigma$ is therefore exactly the minimum-weight perfect matching problem on $G$.
\end{proof}

\begin{figure}[t]
  \centering
  \includegraphics[width=\textwidth]{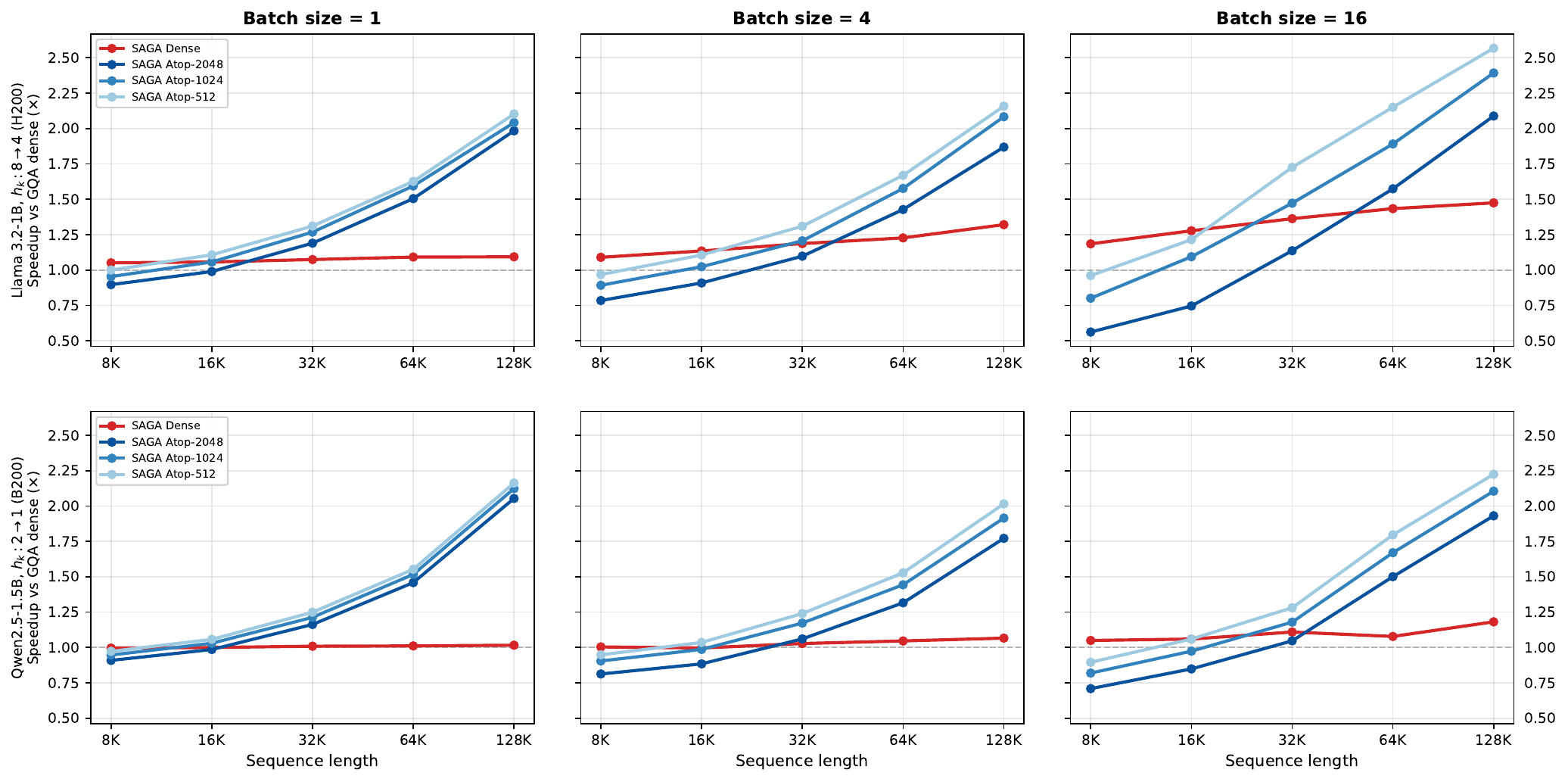}
  \caption{End-to-End speedup over baseline full attention as a function of sequence length, for Llama~3.2-1B ($h_k: 8 \to 4$, H200) and Qwen2.5-1.5B ($h_k: 2 \to 1$, B200) at various batch sizes. Merging alone (SAGA full) gives a consistent speedup across all configurations. Combining merging with \atopn{} attention adds overhead at short sequences but provides substantial gains at long contexts.}
  \label{fig:speedup}
\end{figure}

\section{Experiments}
\label{sec:experiments}

\subsection{Inference decoding speedups}
\label{sec:speedups}

We measure \emph{end-to-end} per-token decode latency on Llama~3.2-1B~\cite{llama3} ($h_k: 8 \to 4$) and Qwen2.5-1.5B~\cite{qwen2.5} ($h_k: 2 \to 1$) using compiled CUDA graphs. We compare the baseline model (original GQA head configuration) against the SAGA model (merged key heads), each evaluated with full attention and with \atopn{} attention at $N \in \{512, 1024, 2048\}$. We sweep sequence lengths from 8K to 128K at different batch sizes, reporting mean per-token latency over 10 runs of 32 generated tokens each.

\cref{fig:speedup} presents end-to-end speedup over baseline full attention. To showcase different GPU architectures, Llama is evaluated on H200 while Qwen is evaluated on the newer B200 (complementary results in \cref{app:additional-results}). Merging key heads alone (SAGA full) provides a modest but consistent speedup across all sequence lengths and batch sizes, reflecting the reduced query-key cost from fewer key heads (see \cref{app:additional-results} for a kernel-level profile, larger models and additional ablations). This benefit grows with batch size as queries compete for memory bandwidth: on Llama at batch size 16, SAGA full attention is up to $1.47\times$ faster than baseline full at 128K context.

Combining merged heads with \atopn{} attention introduces overhead at short sequences due to the threshold calibration and sparse gather, but yields substantial speedups at long contexts. At 128K with batch size 1, SAGA \atopn{} achieves up to a $2.2\times$ end-to-end speedup on Qwen and $2.1\times$ on Llama over baseline full attention. Even at batch size 16, where the \atopn{} overhead is larger, SAGA \atopn{} reaches $2.2\times$ and $2.6\times$ speedups at 128K on Qwen and Llama respectively. 

We note that flash attention \cite{dao2022flashattentionfastmemoryefficientexact} provides little benefit during autoregressive decode, where the attention computation is a matrix-vector product with no quadratic attention matrix to tile. Our full-attention baseline uses faster standard compiled attention, enabling portable support for the asymmetric GQA and ensuring a fair comparison when applying \atopn{} attention.

To assess the quality impact of the \atopn{} approximation, we evaluate all configurations on the full RULER benchmark~\cite{hsieh2024ruler} at context lengths of 8K, 16K, and 32K in \cref{sec:distillation} below. \cref{fig:ruler-accuracy} reports per-task accuracy. On Qwen2.5-1.5B, \atopn{} retains 83--97\% of full attention quality across all context lengths, with most tasks virtually unaffected. The degradation is concentrated in multi-key retrieval tasks, which require the model to simultaneously attend to many dispersed keys -- a setting where aggressive sparsity discards relevant entries.

\subsection{Training from scratch}
\label{sec:scratch}
\begin{figure}[t]
  \centering
  \begin{minipage}[t]{0.49\textwidth}
    \centering
    \includegraphics[width=\linewidth]{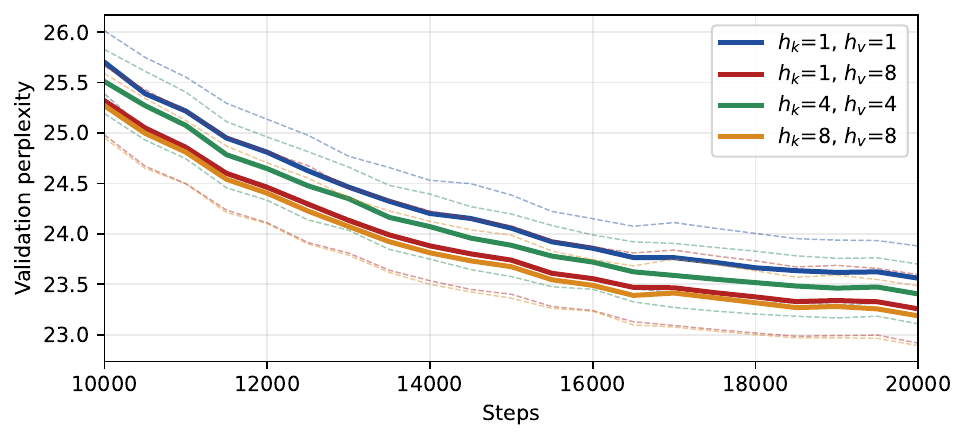}
  \end{minipage}
  \hfill
  \begin{minipage}[t]{0.49\textwidth}
    \centering
    \includegraphics[width=\linewidth]{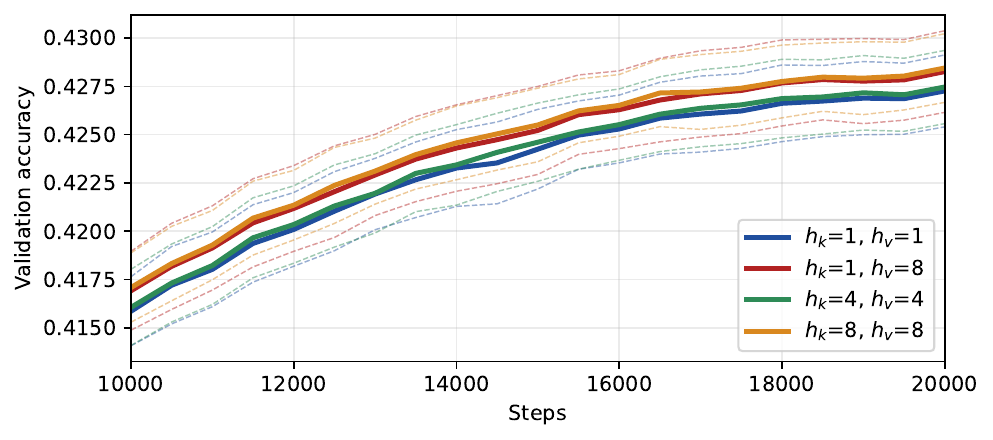}
  \end{minipage}
  \caption{Training from scratch on SmolLM2-360M with $h_q = 16$ query heads. \textbf{Left:} validation perplexity ($\downarrow$). \textbf{Right:} validation accuracy ($\uparrow$). SAGA $(h_k, h_v) = (1, 8)$ closely tracks the high-capacity baseline $(8, 8)$ while outperforming both the matched-key-budget $(1, 1)$ baseline, and the parameter-matched baseline $(4, 4)$. Solid curves are means of 3 seeds; dotted lines indicate $\pm 1$ std.}
  \label{fig:scratch-eval}
\end{figure}
We train SAGA from scratch on the SmolLM2-360M~\citep{allal2025smollm2} backbone, leaving the model unchanged except for the head counts. The number of query heads is fixed to $h_q = 16$, while $h_k$ and $h_v$ are varied. SAGA with $(h_k, h_v) = (1, 8)$ is compared against three GQA baselines: high-capacity with $(h_k, h_v) = (8, 8)$, low-capacity with $(h_k, h_v) = (1, 1)$ matching the key budget of SAGA, and parameter-matched with $(h_k, h_v) = (4, 4)$ roughly matching the parameter count of SAGA. This isolates the contribution of decoupling $h_v$ from $h_k$. Each configuration is run with three seeds. \cref{app:additional-from-scratch} shows that a larger-scale Qwen3-0.6B experiment trained on 50B tokens exhibits the same trend.

\cref{fig:scratch-eval} shows that SAGA $(1, 8)$ closely matches the quality of the $(8, 8)$ baseline despite using a little over half the KV cache of $(8, 8)$. It also clearly outperforms the $(1, 1)$ baseline as well as the $(4, 4)$ baseline despite having a similar number of parameters. This result suggests that the improvement comes from allocating capacity asymmetrically. The same holds for how the key budget itself is spent: a Sigma-style variant~\cite{lin2025sigma} that reduces the key-head dimension instead of reducing the number of heads, $(h_k, d_k) = (2, 32)$, matches SAGA's $(1, 64)$ in KV-cache size and nearly in parameter count (both with $h_v = 8$, $d_v = 64$), yet reaches 41.7\% final validation accuracy versus 42.7\% for SAGA. This is consistent with the results of the SIGMA paper, who favored reducing the number of key heads while leaving the key-head dimension unchanged. 
Combined with \atopn{} attention at inference, this translates to the substantial decoding speedups reported in \cref{sec:speedups} at a small cost in quality. Our from-scratch experiments are necessarily limited in scale, and we hope future work will validate SAGA's benefits in a large-scale pretraining setting.

\subsection{Application to pretrained models}
\label{sec:distillation}

To evaluate SAGA on models with realistic capabilities, we convert pretrained models to the SAGA architecture via knowledge distillation, avoiding the compute and data requirements of training a competitive LLM from scratch. We start from Qwen2.5-1.5B~\cite{qwen2.5}, a highly capable model with only $h_k = 2$ key heads, where the matching from \cref{sec:gqa_to_saga} is trivial (a single pair) and reduces to averaging ($h_k: 2 \to 1$). We then follow up with Llama~3.2-1B~\citep{llama3}, using the full matching-based initialization from \cref{sec:gqa_to_saga} to merge 8 key heads into 4 (see \cref{app:merging-comparison} for a comparison of merging strategies).

Ideally, a fair comparison would increase the student's capacity to compensate for the reduced key heads, for example by increasing the head dimension or adding value heads. However, as teacher and student representations diverge, a short distillation cannot effectively leverage the teacher's signal. For simplicity we keep the student architecture identical to the teacher except for the reduced $h_k$, and accept the slight quality degradation from this conservative approach.

\begin{figure}[t]
  \centering
  \includegraphics[width=\textwidth]{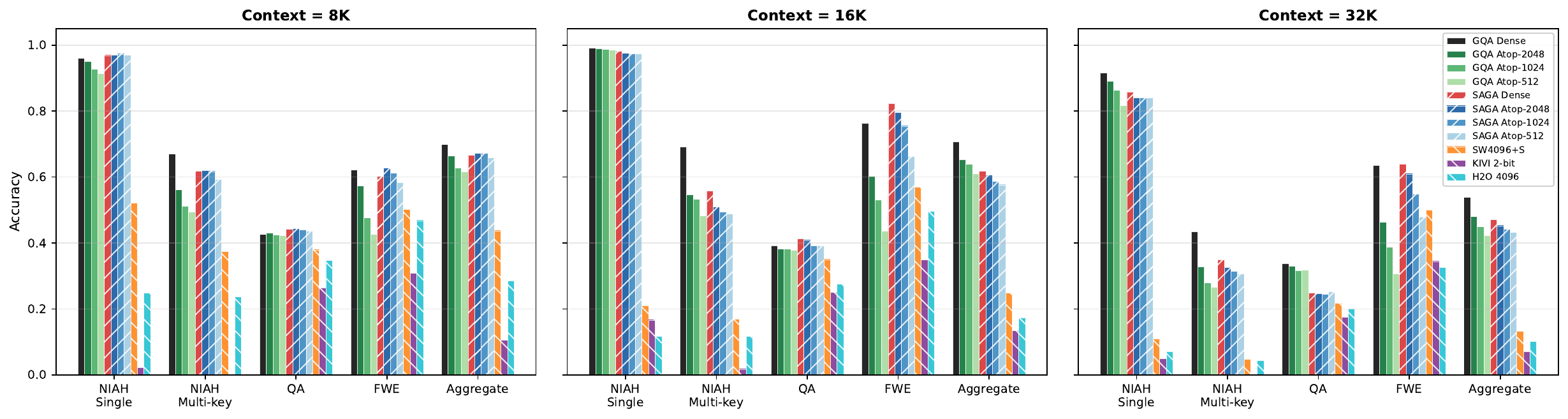}
  \caption{RULER accuracy for full attention and \atopn{} attention at $N \in \{512, 1024, 2048\}$, across context lengths 8K, 16K, and 32K on Qwen2.5-1.5B, for different task categories. SAGA with \atopn{} remains near baseline quality, while alternative eviction-based methods degrade substantially at longer contexts.}
  \label{fig:ruler-accuracy}
\end{figure}

The student is initialized by merging the teacher's key projection heads and trained with a combination of KL divergence from the teacher and hard-label cross-entropy. We train in several stages to expose the student to increasing context lengths while keeping total training short at only around 23K iterations (architecture dependent). We use a mix of SlimPajama~\citep{slimpajama}, QA~\cite{yang2018hotpotqa, rajpurkar2018know} and long-context data~\cite{bai2024longalign, chen2024longlora}, as well as synthetic needle-in-a-haystack examples. We note that the student is heavily weighted to match the teacher's output probabilities on these sequences, so its retained performance reflects architectural capacity rather than task-specific adaptation. Full hyperparameters and data mix details are provided in \cref{app:distillation}. 

The distilled student retains 96.3\% of the teacher's aggregate RULER score across 8K, 16K, and 32K context lengths, with per-task results shown in \cref{fig:ruler-accuracy}. A complementary Llama 3.2-1B comparison is shown in \cref{app:additional-results}. As shown in \cref{tab:distillation-quality}, the SAGA student remains close to the teacher on general-capability benchmarks despite the reduced key head capacity. For comparison, we evaluate alternative decode-acceleration methods: sliding window with attention sinks (SW+S)~\cite{xiao2024efficientstreaminglanguagemodels}, H2O~\cite{zhang2023h2o}, and KIVI~\cite{liu2024kivi}. These methods fall below 50\% aggregate RULER accuracy at 8K context and degrade rapidly at longer contexts, despite the 4096-token budget for sliding window and H2O. We note that SAGA could readily combine with KV cache quantization and eviction methods, such as those used for comparison above.

Furthermore, to experiment at a larger scale, we evaluate Llama3.1-70B-Instruct \citep{llama3} with AtopN only. Since AtopN requires no retraining, these experiments remain feasible with limited computational resources. To highlight that the AtopN approximation holds over long generations, we evaluate on tasks that elicit extended outputs before a final answer: GPQA Main, limited to 8K output tokens, and GPQA Diamond, limited to 16K output tokens \citep{gpqa}. Both consist of graduate-level questions spanning multiple domains, however Diamond is a harder subset. In practice, the model produces long generations on these tasks, averaging over 1.2K output tokens on GPQA Main and 2.1K output tokens on GPQA Diamond. We experiment with $N \in \{2048, 1024, 512\}$ and find that GPQA Main accuracy is unharmed even at $N=512$, whereas GPQA Diamond incurs some degradation. Results are available in Table \ref{tab:atopn_gpqa}

\begin{table}[t]
  \centering
  \caption{General-capability benchmarks for the teacher and SAGA student models.}
  \label{tab:distillation-quality}
  \begin{tabular}{lcccc}
    \toprule
    & \multicolumn{2}{c}{Qwen2.5-1.5B ($h_k: 2 \to 1$)} & \multicolumn{2}{c}{Llama~3.2-1B ($h_k: 8 \to 4$)} \\
    \cmidrule(lr){2-3} \cmidrule(lr){4-5}
    & Baseline & SAGA & Baseline & SAGA \\
    \midrule
    HellaSwag (acc\_norm $\uparrow$) & 0.678 & 0.644 & 0.643 & 0.583 \\
    WikiText (word ppl $\downarrow$) & 12.10 & 12.43 & 11.98 & 13.09 \\
    \bottomrule
  \end{tabular}
\end{table}

\begin{table}[t]
\centering
\caption{Llama3.1-70B-Instruct accuracy on GPQA with AtopN for varying $N$. GPQA Main is limited to 8K output tokens and GPQA Diamond to 16K.}
\label{tab:atopn_gpqa}
\begin{tabular}{lcc}
\toprule
$N$ & GPQA Diamond @16K ($\uparrow$) & GPQA Main @8K ($\uparrow$) \\
\midrule
Baseline & 0.4343 & 0.4241 \\
2048     & 0.3939 & 0.4129 \\
1024     & 0.3838 & 0.4330 \\
512      & 0.3586 & 0.4241 \\
\bottomrule
\end{tabular}
\end{table}

\section{Limitations, discussion and conclusion}

We introduced SAGA, an attention architecture designed around the computational trade-offs of sparse decoding. By decoupling key and value head counts, SAGA directly addresses the query-key bottleneck that dominates under sparse attention, while preserving model capacity through additional value heads. Our experiments show that SAGA nearly matches equivalent GQA models when trained from scratch and retains 96.3\% of the teacher's aggregate RULER score when distilled from pretrained checkpoints, while delivering substantial inference end-to-end speedups of more than $2\times$ at long contexts.

Several limitations point to promising future directions. First, our current implementation relies on the simple \atopn{} method; integrating hardware-aware sparse attention kernels could yield further speedup gains. Second, our distilled SAGA models maintain the same architecture as the teacher aside from the reduced key projections --- future work could compensate for this slight capacity reduction by reallocating parameters, for example by increasing the value head count or dimension. Finally, our quality experiments are limited to models of up to 1.5B parameters; validating SAGA's benefits in a full-scale pretraining setting remains an important open question. More broadly, we believe that co-designing attention architectures with their inference-time sparsity patterns is a fertile direction, and hope that SAGA serves as a step toward making this a standard consideration in LLM design.

\begin{ack}

The research of DS was funded by the European Union (ERC, A-B-C-Deep, 101039436). Views and opinions expressed are however those of the author only and do not necessarily reflect those of the European Union or the European Research Council Executive Agency (ERCEA). Neither the European Union nor the granting authority can be held responsible for them. The research of DS was also supported by the Council for Higher Education in Israel under the Moonshot Project. 

\end{ack}

\clearpage

\bibliographystyle{unsrtnat}
\bibliography{refs}

@article{li2024snapkv,
  title={Snapkv: Llm knows what you are looking for before generation},
  author={Li, Yuhong and Huang, Yingbing and Yang, Bowen and Venkitesh, Bharat and Locatelli, Acyr and Ye, Hanchen and Cai, Tianle and Lewis, Patrick and Chen, Deming},
  journal={Advances in Neural Information Processing Systems},
  volume={37},
  pages={22947--22970},
  year={2024}
}

@article{vaswani2017attention,
  title={Attention is all you need},
  author={Vaswani, Ashish and Shazeer, Noam and Parmar, Niki and Uszkoreit, Jakob and Jones, Llion and Gomez, Aidan N and Kaiser, {\L}ukasz and Polosukhin, Illia},
  journal={Advances in neural information processing systems},
  volume={30},
  year={2017}
}

@article{pope2023efficiently,
  title={Efficiently scaling transformer inference},
  author={Pope, Reiner and Douglas, Sholto and Chowdhery, Aakanksha and Devlin, Jacob and Bradbury, James and Heek, Jonathan and Xiao, Kefan and Agrawal, Shivani and Dean, Jeff},
  journal={Proceedings of Machine Learning and Systems},
  volume={5},
  pages={606--624},
  year={2023}
}

@misc{qwen2.5,
    title = {Qwen2.5: A Party of Foundation Models},
    url = {https://qwenlm.github.io/blog/qwen2.5/},
    author = {Qwen Team},
    month = {September},
    year = {2024}
}

@article{qwen2,
      title={Qwen2 Technical Report}, 
      author={An Yang and Baosong Yang and Binyuan Hui and Bo Zheng and Bowen Yu and Chang Zhou and Chengpeng Li and Chengyuan Li and Dayiheng Liu and Fei Huang and Guanting Dong and Haoran Wei and Huan Lin and Jialong Tang and Jialin Wang and Jian Yang and Jianhong Tu and Jianwei Zhang and Jianxin Ma and Jin Xu and Jingren Zhou and Jinze Bai and Jinzheng He and Junyang Lin and Kai Dang and Keming Lu and Keqin Chen and Kexin Yang and Mei Li and Mingfeng Xue and Na Ni and Pei Zhang and Peng Wang and Ru Peng and Rui Men and Ruize Gao and Runji Lin and Shijie Wang and Shuai Bai and Sinan Tan and Tianhang Zhu and Tianhao Li and Tianyu Liu and Wenbin Ge and Xiaodong Deng and Xiaohuan Zhou and Xingzhang Ren and Xinyu Zhang and Xipin Wei and Xuancheng Ren and Yang Fan and Yang Yao and Yichang Zhang and Yu Wan and Yunfei Chu and Yuqiong Liu and Zeyu Cui and Zhenru Zhang and Zhihao Fan},
      journal={arXiv preprint arXiv:2407.10671},
      year={2024}
}

@article{gupta2021memory,
  title={Memory-efficient Transformers via Top-$ k $ Attention},
  author={Gupta, Ankit and Dar, Guy and Goodman, Shaya and Ciprut, David and Berant, Jonathan},
  journal={arXiv preprint arXiv:2106.06899},
  year={2021}
}

@inproceedings{wang2022kvt,
  title={Kvt: k-nn attention for boosting vision transformers},
  author={Wang, Pichao and Wang, Xue and Wang, Fan and Lin, Ming and Chang, Shuning and Li, Hao and Jin, Rong},
  booktitle={European conference on computer vision},
  pages={285--302},
  year={2022},
  organization={Springer}
}

@article{wu2022memorizing,
  title={Memorizing transformers},
  author={Wu, Yuhuai and Rabe, Markus N and Hutchins, DeLesley and Szegedy, Christian},
  journal={arXiv preprint arXiv:2203.08913},
  year={2022}
}

@article{bertsch2023unlimiformer,
  title={Unlimiformer: Long-range transformers with unlimited length input},
  author={Bertsch, Amanda and Alon, Uri and Neubig, Graham and Gormley, Matthew},
  journal={Advances in Neural Information Processing Systems},
  volume={36},
  pages={35522--35543},
  year={2023}
}

@article{haris2024k,
  title={$ k $ NN Attention Demystified: A Theoretical Exploration for Scalable Transformers},
  author={Haris, Themistoklis},
  journal={arXiv preprint arXiv:2411.04013},
  year={2024}
}

@article{lou2024sparser,
  title={Sparser is faster and less is more: Efficient sparse attention for long-range transformers},
  author={Lou, Chao and Jia, Zixia and Zheng, Zilong and Tu, Kewei},
  journal={arXiv preprint arXiv:2406.16747},
  year={2024}
}

@article{hsieh2024ruler,
  title={RULER: What's the Real Context Size of Your Long-Context Language Models?},
  author={Hsieh, Cheng-Ping and Sun, Simeng and Kriman, Samuel and Acharya, Shantanu and Rekesh, Dima and Jia, Fei and Zhang, Yang and Ginsburg, Boris},
  journal={arXiv preprint arXiv:2404.06654},
  year={2024}
}

@misc{dao2022flashattentionfastmemoryefficientexact,
      title={FlashAttention: Fast and Memory-Efficient Exact Attention with IO-Awareness}, 
      author={Tri Dao and Daniel Y. Fu and Stefano Ermon and Atri Rudra and Christopher Ré},
      year={2022},
      eprint={2205.14135},
      archivePrefix={arXiv},
      primaryClass={cs.LG},
      url={https://arxiv.org/abs/2205.14135}, 
}

@misc{xiao2024efficientstreaminglanguagemodels,
      title={Efficient Streaming Language Models with Attention Sinks}, 
      author={Guangxuan Xiao and Yuandong Tian and Beidi Chen and Song Han and Mike Lewis},
      year={2024},
      eprint={2309.17453},
      archivePrefix={arXiv},
      primaryClass={cs.CL},
      url={https://arxiv.org/abs/2309.17453}, 
}

@misc{kitaev2020reformerefficienttransformer,
      title={Reformer: The Efficient Transformer}, 
      author={Nikita Kitaev and Łukasz Kaiser and Anselm Levskaya},
      year={2020},
      eprint={2001.04451},
      archivePrefix={arXiv},
      primaryClass={cs.LG},
      url={https://arxiv.org/abs/2001.04451}, 
}

@article{radford2019language,
  title={Language models are unsupervised multitask learners},
  author={Radford, Alec and Wu, Jeffrey and Child, Rewon and Luan, David and Amodei, Dario and Sutskever, Ilya and others},
  journal={OpenAI blog},
  volume={1},
  number={8},
  pages={9},
  year={2019}
}

@article{brown2020language,
  title={Language models are few-shot learners},
  author={Brown, Tom and Mann, Benjamin and Ryder, Nick and Subbiah, Melanie and Kaplan, Jared D and Dhariwal, Prafulla and Neelakantan, Arvind and Shyam, Pranav and Sastry, Girish and Askell, Amanda and others},
  journal={Advances in neural information processing systems},
  volume={33},
  pages={1877--1901},
  year={2020}
}

@article{touvron2023llama,
  title={Llama: Open and efficient foundation language models},
  author={Touvron, Hugo and Lavril, Thibaut and Izacard, Gautier and Martinet, Xavier and Lachaux, Marie-Anne and Lacroix, Timoth{\'e}e and Rozi{\`e}re, Baptiste and Goyal, Naman and Hambro, Eric and Azhar, Faisal and others},
  journal={arXiv preprint arXiv:2302.13971},
  year={2023}
}

@article{wang2019superglue,
  title={Superglue: A stickier benchmark for general-purpose language understanding systems},
  author={Wang, Alex and Pruksachatkun, Yada and Nangia, Nikita and Singh, Amanpreet and Michael, Julian and Hill, Felix and Levy, Omer and Bowman, Samuel},
  journal={Advances in neural information processing systems},
  volume={32},
  year={2019}
}

@inproceedings{papineni2002bleu,
  title={Bleu: a method for automatic evaluation of machine translation},
  author={Papineni, Kishore and Roukos, Salim and Ward, Todd and Zhu, Wei-Jing},
  booktitle={Proceedings of the 40th annual meeting of the Association for Computational Linguistics},
  pages={311--318},
  year={2002}
}

@article{wei2022chain,
  title={Chain-of-thought prompting elicits reasoning in large language models},
  author={Wei, Jason and Wang, Xuezhi and Schuurmans, Dale and Bosma, Maarten and Xia, Fei and Chi, Ed and Le, Quoc V and Zhou, Denny and others},
  journal={Advances in neural information processing systems},
  volume={35},
  pages={24824--24837},
  year={2022}
}

@inproceedings{shen2021efficient,
  title={Efficient attention: Attention with linear complexities},
  author={Shen, Zhuoran and Zhang, Mingyuan and Zhao, Haiyu and Yi, Shuai and Li, Hongsheng},
  booktitle={Proceedings of the IEEE/CVF winter conference on applications of computer vision},
  pages={3531--3539},
  year={2021}
}

@inproceedings{keles2023computational,
  title={On the computational complexity of self-attention},
  author={Keles, Feyza Duman and Wijewardena, Pruthuvi Mahesakya and Hegde, Chinmay},
  booktitle={International conference on algorithmic learning theory},
  pages={597--619},
  year={2023},
  organization={PMLR}
}

@article{ainslie2023gqa,
  title={Gqa: Training generalized multi-query transformer models from multi-head checkpoints},
  author={Ainslie, Joshua and Lee-Thorp, James and De Jong, Michiel and Zemlyanskiy, Yury and Lebr{\'o}n, Federico and Sanghai, Sumit},
  journal={arXiv preprint arXiv:2305.13245},
  year={2023}
}

@article{zaheer2020big,
  title={Big bird: Transformers for longer sequences},
  author={Zaheer, Manzil and Guruganesh, Guru and Dubey, Kumar Avinava and Ainslie, Joshua and Alberti, Chris and Ontanon, Santiago and Pham, Philip and Ravula, Anirudh and Wang, Qifan and Yang, Li and others},
  journal={Advances in neural information processing systems},
  volume={33},
  pages={17283--17297},
  year={2020}
}

@article{beltagy2020longformer,
  title={Longformer: The long-document transformer},
  author={Beltagy, Iz and Peters, Matthew E and Cohan, Arman},
  journal={arXiv preprint arXiv:2004.05150},
  year={2020}
}

@article{child2019generating,
  title={Generating long sequences with sparse transformers},
  author={Child, Rewon},
  journal={arXiv preprint arXiv:1904.10509},
  year={2019}
}

@article{zhang2023h2o,
  title={H2o: Heavy-hitter oracle for efficient generative inference of large language models},
  author={Zhang, Zhenyu and Sheng, Ying and Zhou, Tianyi and Chen, Tianlong and Zheng, Lianmin and Cai, Ruisi and Song, Zhao and Tian, Yuandong and R{\'e}, Christopher and Barrett, Clark and others},
  journal={Advances in Neural Information Processing Systems},
  volume={36},
  pages={34661--34710},
  year={2023}
}

@article{tang2024quest,
  title={Quest: Query-aware sparsity for efficient long-context llm inference},
  author={Tang, Jiaming and Zhao, Yilong and Zhu, Kan and Xiao, Guangxuan and Kasikci, Baris and Han, Song},
  journal={arXiv preprint arXiv:2406.10774},
  year={2024}
}

@article{liu2024minicache,
  title={Minicache: Kv cache compression in depth dimension for large language models},
  author={Liu, Akide and Liu, Jing and Pan, Zizheng and He, Yefei and Haffari, Gholamreza and Zhuang, Bohan},
  journal={Advances in Neural Information Processing Systems},
  volume={37},
  pages={139997--140031},
  year={2024}
}

@inproceedings{sheng2023flexgen,
  title={Flexgen: High-throughput generative inference of large language models with a single gpu},
  author={Sheng, Ying and Zheng, Lianmin and Yuan, Binhang and Li, Zhuohan and Ryabinin, Max and Chen, Beidi and Liang, Percy and R{\'e}, Christopher and Stoica, Ion and Zhang, Ce},
  booktitle={International Conference on Machine Learning},
  pages={31094--31116},
  year={2023},
  organization={PMLR}
}

@inproceedings{zhangspargeattention,
  title={SpargeAttention: Accurate and Training-free Sparse Attention Accelerating Any Model Inference},
  author={Zhang, Jintao and Xiang, Chendong and Huang, Haofeng and Xi, Haocheng and Zhu, Jun and Chen, Jianfei and others},
  booktitle={Forty-second International Conference on Machine Learning},
  year={2025}
}

@article{yuan2025blasst,
  title={BLASST: Dynamic BLocked Attention Sparsity via Softmax Thresholding},
  author={Yuan, Jiayi and Shinn, Cameron and Xu, Kai and Cui, Jingze and Klimiashvili, George and Xiao, Guangxuan and Zheng, Perkz and Li, Bo and Zhou, Yuxin and Ye, Zhouhai and others},
  journal={arXiv preprint arXiv:2512.12087},
  year={2025}
}

@article{ohayon2025block,
  title={Block Sparse Flash Attention},
  author={Ohayon, Daniel and Lamprecht, Itay and Hubara, Itay and Cohen, Israel and Soudry, Daniel and Elata, Noam},
  journal={arXiv preprint arXiv:2512.07011},
  year={2025}
}

@misc{slimpajama,
  author = {Soboleva, Daria and Al-Khateeb, Faisal and Myers, Robert and Steeves, Jacob R. and Hestness, Joel and Dey, Nolan},
  title = {SlimPajama: A 627B token cleaned and deduplicated version of RedPajama},
  month = jun,
  year = {2023},
  howpublished = {\url{https://www.cerebras.ai/blog/slimpajama-a-627b-token-cleaned-and-deduplicated-version-of-redpajama}},
  url = {https://huggingface.co/datasets/cerebras/SlimPajama-627B}
}

@inproceedings{loshchilov2019decoupled,
  title={Decoupled Weight Decay Regularization},
  author={Loshchilov, Ilya and Hutter, Frank},
  booktitle={International Conference on Learning Representations},
  year={2019}
}

@article{allal2025smollm2,
  title={{SmolLM2}: When Smol Goes Big -- Data-Centric Training of a Small Language Model},
  author={Allal, Loubna Ben and Lozhkov, Anton and Bakouch, Elie and Penedo, Guilherme and Tazi, Nouamane and Tunstall, Lewis and others},
  journal={arXiv preprint arXiv:2502.02737},
  year={2025}
}

@inproceedings{liu2024kivi,
  title={{KIVI}: A Tuning-Free Asymmetric 2bit Quantization for {KV} Cache},
  author={Liu, Zirui and Yuan, Jiayi and Jin, Hongye and Zhong, Shaochen and Xu, Zhaozhuo and Braverman, Vladimir and Chen, Beidi and Hu, Xia},
  booktitle={Proceedings of the 41st International Conference on Machine Learning},
  year={2024},
  eprint={2402.02750},
  archivePrefix={arXiv},
  primaryClass={cs.CL}
}

@inproceedings{chen2024dha,
  title={{DHA}: Learning Decoupled-Head Attention from Transformer Checkpoints via Adaptive Heads Fusion},
  author={Chen, Yilong and Zhang, Linhao and Shang, Junyuan and Zhang, Zhenyu and Liu, Tingwen and Wang, Shuohuan and Sun, Yu},
  booktitle={Advances in Neural Information Processing Systems},
  volume={37},
  year={2024}
}

@misc{llama3,
      title={The Llama 3 Herd of Models},
      author={Aaron Grattafiori and Abhimanyu Dubey and others},
      year={2024},
      eprint={2407.21783},
      archivePrefix={arXiv},
      primaryClass={cs.AI},
      url={https://arxiv.org/abs/2407.21783},
}

@inproceedings{yang2018hotpotqa,
  title={HotpotQA: A Dataset for Diverse, Explainable Multi-hop Question Answering},
  author={Yang, Zhilin and Qi, Peng and Zhang, Saizheng and Bengio, Yoshua and Cohen, William W. and Salakhutdinov, Ruslan and Manning, Christopher D.},
  booktitle={Proceedings of the 2018 Conference on Empirical Methods in Natural Language Processing},
  pages={2369--2380},
  year={2018},
  url={https://arxiv.org/abs/1809.09600},
}

@inproceedings{rajpurkar2018know,
  title={Know What You Don't Know: Unanswerable Questions for {SQuAD}},
  author={Rajpurkar, Pranav and Jia, Robin and Liang, Percy},
  booktitle={Proceedings of the 56th Annual Meeting of the Association for Computational Linguistics},
  pages={784--789},
  year={2018},
  url={https://arxiv.org/abs/1806.03822},
}

@misc{bai2024longalign,
      title={LongAlign: A Recipe for Long Context Alignment of Large Language Models},
      author={Yushi Bai and Xin Lv and Jiajie Zhang and Yuze He and Ji Qi and Lei Hou and Jie Tang and Yuxiao Dong and Juanzi Li},
      year={2024},
      eprint={2401.18058},
      archivePrefix={arXiv},
      primaryClass={cs.CL},
      url={https://arxiv.org/abs/2401.18058},
}

@misc{chen2024longlora,
      title={LongLoRA: Efficient Fine-tuning of Long-Context Large Language Models},
      author={Yukang Chen and Shengju Qian and Haotian Tang and Xin Lai and Zhijian Liu and Song Han and Jiaya Jia},
      year={2024},
      eprint={2309.12307},
      archivePrefix={arXiv},
      primaryClass={cs.CL},
      url={https://arxiv.org/abs/2309.12307},
}

@misc{yuan2025nsa,
      title={Native Sparse Attention: Hardware-Aligned and Natively Trainable Sparse Attention},
      author={Jingyang Yuan and Huazuo Gao and Damai Dai and Junyu Luo and Liang Zhao and Zhengyan Zhang and Zhenda Xie and Y. X. Wei and Lean Wang and Zhiping Xiao and Yuqing Wang and Chong Ruan and Ming Zhang and Wenfeng Liang and Wangding Zeng},
      year={2025},
      eprint={2502.11089},
      archivePrefix={arXiv},
      primaryClass={cs.CL},
      url={https://arxiv.org/abs/2502.11089},
}

@misc{lin2025sigma,
      title={Sigma: Differential Rescaling of Query, Key and Value for Efficient Language Models},
      author={Zhenghao Lin and Zihao Tang and Xiao Liu and Yeyun Gong and Yi Chang and Jian Jia and Jian Li and Yelong Shen and Chen Lin and Nan Duan and Weizhu Chen},
      year={2025},
      eprint={2501.13629},
      archivePrefix={arXiv},
      primaryClass={cs.CL},
      url={https://arxiv.org/abs/2501.13629},
}

@misc{davies2025liminalexploringfrontiersllm,
      title={LIMINAL: Exploring The Frontiers of LLM Decode Performance}, 
      author={Michael Davies and Neal Crago and Karthikeyan Sankaralingam and Christos Kozyrakis},
      year={2025},
      eprint={2507.14397},
      archivePrefix={arXiv},
      primaryClass={cs.AR},
      url={https://arxiv.org/abs/2507.14397}, 
}

@article{liu2024deepseek,
  title={Deepseek-v2: A strong, economical, and efficient mixture-of-experts language model},
  author={Liu, Aixin and Feng, Bei and Wang, Bin and Wang, Bingxuan and Liu, Bo and Zhao, Chenggang and Dengr, Chengqi and Ruan, Chong and Dai, Damai and Guo, Daya and others},
  journal={arXiv preprint arXiv:2405.04434},
  year={2024}
}

@misc{deepseekv32,
  title  = {{DeepSeek-V3.2-Exp}: Boosting Long-Context Efficiency with {DeepSeek} Sparse Attention},
  author = {{DeepSeek-AI}},
  year   = {2025},
  howpublished = {\url{https://github.com/deepseek-ai/DeepSeek-V3.2-Exp}}
}

@misc{gpqa,
      title={GPQA: A Graduate-Level Google-Proof Q\&A Benchmark}, 
      author={David Rein and Betty Li Hou and Asa Cooper Stickland and Jackson Petty and Richard Yuanzhe Pang and Julien Dirani and Julian Michael and Samuel R. Bowman},
      year={2023},
      eprint={2311.12022},
      archivePrefix={arXiv},
      primaryClass={cs.AI},
      url={https://arxiv.org/abs/2311.12022}, 
}


\clearpage
\appendix

\section{Proof of Lemma~\ref{lem:pair_logit_merge}}
\label{app:logit_merge}

Recall that $\mathcal{Q}_r$ denotes the query heads that share $\rmW_K^r$. Define the realized query matrices $\rmQ^j := \rmX\rmW_Q^j$, and fix a pair of teacher key heads $a,b \in \{1,\ldots,2h_k^S\}$. Recall the pairwise loss
\begin{equation}
\begin{aligned}
\mathcal{L}_{a,b}(\rmM)
:=
\sum_{j \in \mathcal{Q}_a}
&\left\lVert
\rmQ^j(\rmX\rmM)^\top
-
\rmQ^j(\rmX\rmW_K^a)^\top
\right\rVert_{\mathrm{F}}^2 \\
+
\sum_{j \in \mathcal{Q}_b}
&\left\lVert
\rmQ^j(\rmX\rmM)^\top
-
\rmQ^j(\rmX\rmW_K^b)^\top
\right\rVert_{\mathrm{F}}^2.
\end{aligned}
\end{equation}

\begin{proof}[Proof of Lemma~\ref{lem:pair_logit_merge}]
For any fixed $j \in \mathcal{Q}_a$, we have
\begin{equation}
\begin{aligned}
&\left\lVert
\rmQ^j(\rmX\rmM)^\top
-
\rmQ^j(\rmX\rmW_K^a)^\top
\right\rVert_{\mathrm{F}}^2
=
\left\lVert
\rmQ^j\big(\rmX(\rmM-\rmW_K^a)\big)^\top
\right\rVert_{\mathrm{F}}^2 \\
&=
\operatorname{tr}\!\left(
\rmQ^j\big(\rmX(\rmM-\rmW_K^a)\big)^\top
\rmX(\rmM-\rmW_K^a)
(\rmQ^j)^\top
\right) \\
&=
\operatorname{tr}\!\left(
(\rmM-\rmW_K^a)^\top
\rmX^\top \rmX
(\rmM-\rmW_K^a)
(\rmQ^j)^\top(\rmQ^j)
\right),
\end{aligned}
\end{equation}
where we used the trace form of the Frobenius norm, $\lVert A \rVert_{\mathrm{F}}^2 = \operatorname{tr}(AA^\top)$, and the cyclic property of the trace. Denoting
\begin{equation}
\rmS_a
:=
\sum_{j \in \mathcal{Q}_a}
(\rmQ^j)^\top(\rmQ^j),
\qquad
\rmS_b
:=
\sum_{j \in \mathcal{Q}_b}
(\rmQ^j)^\top(\rmQ^j)
\end{equation}
and summing over $j \in \mathcal{Q}_a$ gives
\begin{equation}
\sum_{j \in \mathcal{Q}_a}
\left\lVert
\rmQ^j(\rmX\rmM)^\top
-
\rmQ^j(\rmX\rmW_K^a)^\top
\right\rVert_{\mathrm{F}}^2
=
\operatorname{tr}\!\left(
(\rmM-\rmW_K^a)^\top \rmX^\top \rmX (\rmM-\rmW_K^a)\rmS_a
\right).
\end{equation}
Applying the same argument to the second summation term, $\mathcal{L}_{a,b}(\rmM)$ is equivalent to minimizing
\begin{equation}
\operatorname{tr}\!\left((\rmM-\rmW_K^a)^\top \rmX^\top \rmX (\rmM-\rmW_K^a)\rmS_a\right)
+
\operatorname{tr}\!\left((\rmM-\rmW_K^b)^\top \rmX^\top \rmX (\rmM-\rmW_K^b)\rmS_b\right).
\end{equation}
This is a convex quadratic function of $\rmM$. We find the minimizer by taking the derivative with respect to $\rmM$. Using the identity $\frac{\partial}{\partial \rmM}\operatorname{tr}\!\left(\rmM^\top \rmB \rmM \rmC\right) = 2\rmB\rmM\rmC$ for symmetric positive semidefinite $\rmB$ and $\rmC$, and applying it with $\rmM \leftarrow \rmM - \rmW_K^a$, $\rmB \leftarrow \rmX^\top\rmX$, and $\rmC \leftarrow \rmS_a$ (and analogously for the second term), the gradient of the objective with respect to $\rmM$ is
\begin{equation}
2\rmX^\top\rmX(\rmM - \rmW_K^a)\rmS_a + 2\rmX^\top\rmX(\rmM - \rmW_K^b)\rmS_b.
\end{equation}
Setting this to zero gives
\begin{equation}
\rmX^\top \rmX
\Big(
\rmM(\rmS_a+\rmS_b)
-
\rmW_K^a\rmS_a
-
\rmW_K^b\rmS_b
\Big)
=
0.
\end{equation}
Hence any $\rmM$ satisfying
\begin{equation}
\rmM(\rmS_a+\rmS_b)
=
\rmW_K^a\rmS_a + \rmW_K^b\rmS_b
\end{equation}
is optimal. One such solution is
\begin{equation}
\rmM
=
\left(\rmW_K^a\rmS_a + \rmW_K^b\rmS_b\right)
\left(\rmS_a+\rmS_b\right)^{\dagger}.
\end{equation}
Since $\rmS_a$ and $\rmS_b$ are positive semidefinite, the rows of $\rmW_K^a\rmS_a + \rmW_K^b\rmS_b$ lie in the row space of $\rmS_a+\rmS_b$, so this choice indeed satisfies the linear system above.
\end{proof}

\section{Induced permutations for query, value, and output projections}
\label{app:permute_qvo}

Once a matching has been chosen, let $\sigma=(\sigma_1,\ldots,\sigma_{2h_k^S})$ be any permutation of $\{1,\ldots,2h_k^S\}$ whose consecutive pairs are exactly the matched pairs, i.e.\ $(\sigma_{2i-1},\sigma_{2i})$ for $i=1,\ldots,h_k^S$. The purpose of the induced permutations is to re-index the teacher heads before we merge its two key heads so that each matched pair becomes consecutive.

We first permute the key and value heads by $\sigma$:
\begin{equation}
\widetilde{\rmW}_K
:=
\big[
\rmW_K^{\sigma_1}\;\; \cdots \;\; \rmW_K^{\sigma_{2h_k^S}}
\big],
\qquad
\widetilde{\rmW}_V
:=
\big[
\rmW_V^{\sigma_1}\;\; \cdots \;\; \rmW_V^{\sigma_{2h_k^S}}
\big].
\end{equation}
For the query heads, let
\begin{equation}
g := \frac{h_q}{h_k^T} = \frac{h_q}{2h_k^S},
\qquad
B_r := \big((r-1)g+1,\ldots,rg\big), \quad r=1,\ldots,2h_k^S,
\end{equation}
so $B_r$ is the block of query heads that share the teacher key head $r$. The associated query-head permutation $\pi_Q$ is the concatenation of these blocks into a single sequence of length $h_q$:
\begin{equation}
\pi_Q := \big(B_{\sigma_1}\Vert B_{\sigma_2}\Vert\cdots\Vert B_{\sigma_{2h_k^S}}\big) \in \{1,\ldots,h_q\}^{h_q},
\end{equation}
where $\Vert$ denotes concatenation, so that $\pi_Q(k)$ for $k=1,\ldots,h_q$ refers to the $k$-th entry of this flattened sequence.
Writing the query projection in head blocks,
\begin{equation}
\rmW_Q
=
\big[
\rmW_Q^1\;\; \cdots \;\; \rmW_Q^{h_q}
\big],
\end{equation}
the permuted query projection is
\begin{equation}
\widetilde{\rmW}_Q
:=
\big[
\rmW_Q^{\pi_Q(1)}\;\; \cdots \;\; \rmW_Q^{\pi_Q(h_q)}
\big].
\end{equation}

For the output projection, write $\rmW_O$ as a vertical stack of query-head input blocks,
\begin{equation}
\rmW_O
=
\begin{bmatrix}
\rmW_{O,1} \\
\vdots \\
\rmW_{O,h_q}
\end{bmatrix},
\qquad
\rmW_{O,r} \in \R^{d \times D}.
\end{equation}
We permute these input blocks by the same query-head permutation:
\begin{equation}
\widetilde{\rmW}_O
:=
\begin{bmatrix}
\rmW_{O,\pi_Q(1)} \\
\vdots \\
\rmW_{O,\pi_Q(h_q)}
\end{bmatrix}.
\end{equation}

Before merging, this re-indexing leaves the attention module unchanged. After merging, the two value heads at positions $2i-1$ and $2i$ and their associated $2g$ query heads are exactly the heads that should share the merged SAGA key head $i$.

\subsection{Merging strategies ablation}
\label{app:merging-comparison}

We compare three strategies for initializing the merged SAGA key projections on Llama~3.2-1B, measuring cross-entropy loss and per-token KL divergence from the teacher immediately after initialization (before distillation):
\begin{itemize}
    \item \textbf{Sequential pairing}: merges consecutive head pairs in their original tensor order (effectively arbitrary pairing) by averaging their key projection matrices.
    \item \textbf{Key matching}: minimum-weight perfect matching with edge weight
    \begin{equation}
    w(a,b) := \big\lVert \rmX\rmW_K^a - \rmX\rmW_K^b \big\rVert_{\mathrm{F}}^2,
    \end{equation}
    i.e.\ the same matching procedure as ours, but the graph weights are based only on the realized key matrices instead of attention logits. Here, the minimizer is the average $\tfrac{1}{2}(\rmW_K^a + \rmW_K^b)$.
    \item \textbf{Logit matching} (ours): the minimum-weight perfect matching on attention logit reconstruction error described in \cref{sec:init_saga}.
\end{itemize}

\begin{table}[h]
  \centering
  \caption{Initialization quality before distillation (Llama~3.2-1B). Lower is better for both metrics.}
  \label{tab:merging-comparison}
  \begin{tabular}{llcc}
    \toprule
    Reduction & Method & CE & KL \\
    \midrule
    $8 \to 4$ & Sequential pairing & 7.231 & 4.754 \\
    $8 \to 4$ & Key matching & 7.727 & 5.254 \\
    $8 \to 4$ & Logit matching (ours) & \textbf{6.748} & \textbf{4.271} \\
    \midrule
    $8 \to 2$ & Sequential pairing & 10.306 & 7.819 \\
    $8 \to 2$ & Key matching & 9.077 & 6.596 \\
    $8 \to 2$ & Logit matching (ours) & \textbf{7.597} & \textbf{5.125} \\
    \bottomrule
  \end{tabular}
\end{table}

Logit-matching initialization consistently achieves the lowest reconstruction error across both reduction ratios, with the advantage growing as the reduction becomes more aggressive ($8 \to 2$). The two other merging strategies do not have a clear superior method. Sequential pairing is better at $8\!\to\!4$, and key matching at $8\!\to\!2$.

\section{Implementation details}
\label{app:implementation}

Code implementation of our experiments is included in \url{https://github.com/noamelata/SAGA}.

\subsection{\atopn{} attention}
\label{app:topn-impl}

Our \atopn{} attention is an efficient implementation of the algorithm described in \cref{sec:top-n-attn}. Specifically, it uses a pre-calibrated per-head threshold to collect all entries that exceed it, greatly improving runtimes. This comparison operation uses a custom portable Triton\footnote{\url{https://github.com/triton-lang/triton}, MIT license.} kernel, replacing pytorch's CUDA-optimized \texttt{torch.topk} used for true top-N values.

During the first decode step after prefill, we compute exact attention probabilities via softmax and extract the $N$-th largest value per head using \texttt{torch.topk}. This value serves as the threshold for the subsequent $n$ decode steps. Because the attention distribution changes slowly across consecutive decode steps, we use $n=32$, thus amortizing the computational cost of \texttt{torch.topk} across 32 decode steps.

Alongside with the reduced probability-value multiplication, computing softmax only on the top-N values is a significant source of speedup. Yet, we find that thresholds in logit space tend to perform poorly. Instead, we make use of a probability-space threshold $\tau$ transformed into the logit space, a much cheaper but mathematically equivalent operation. Given raw attention logits $\mathbf{z}$, we approximate the max logit $\hat{m}$ from a small tail of positions and compute the log-sum-exp denominator $\hat{s} = \sum_i \exp(z_i - \hat{m})$. The logit-space threshold is then $\hat{m} + \log(\tau \cdot \hat{s})$, which requires only a single reduction pass over the full sequence rather than a full softmax.

The selected indices are used to gather the corresponding value vectors, after which the logit-space weights are normalized via softmax (the $-\infty$ padding receives zero probability) and multiplied with the gathered values via regular batched matrix multiplication (for batch < 16) or sparse batched matrix multiplication (otherwise) to produce the attention output.

Certain layers can be designated as exempt from Top-$N$ selection and always use full attention. In our experiments, we exempt the first and 14th attention layers for Qwen2.5-1.5B\footnote{\url{https://huggingface.co/Qwen/Qwen2.5-1.5B}, Apache 2.0 license.} and the first layer for Llama~3.2-1B\footnote{\url{https://huggingface.co/meta-llama/Llama-3.2-1B}, Llama 3.2 Community License.}, as these layers tend to have more diffuse attention patterns.

The threshold comparison and index collection are implemented as a single Triton kernel. Each row of the attention weight matrix is partitioned into blocks (autotuned over 512, 1024, and 2048 elements). Within each block, elements exceeding the threshold are identified via a comparison mask, their positions computed via an intra-block cumulative sum, and results written to an output buffer using an atomic counter to coordinate across blocks. The kernel outputs both the selected values and their indices into a fixed-size buffer of capacity $2N$, with unused slots initialized to $-\infty$. If the number of entries exceeding the threshold exceeds $2N$, the atomic counter saturates and excess entries are silently dropped, preserving memory safety. The buffer is allocated once and reused across steps, ensuring compatibility with CUDA graph capture.

\subsection{Training SAGA from scratch}
\label{app:scratch-impl}

We use the SmolLM2-360M\footnote{\url{https://huggingface.co/HuggingFaceTB/SmolLM2-360M}, Apache 2.0 license.} backbone. All models are trained on SlimPajama\footnote{\url{https://huggingface.co/datasets/cerebras/SlimPajama-627B}, Apache 2.0 license.} for 20K steps with batch size 256 and sequence length 1024, totaling $\sim$5.24B tokens. We use the AdamW optimizer~\citep{loshchilov2019decoupled} with $(\beta_1, \beta_2) = (0.9, 0.999)$ and weight decay 0.1 and a cosine learning rate schedule with 500 warmup steps, peak learning rate $3 \times 10^{-4}$, and minimum learning rate $7.5 \times 10^{-5}$. Training is performed in bfloat16, and each configuration is run with 3 seeds.

\subsection{Distillation to SAGA}
\label{app:distillation}

The student model is initialized from the teacher by merging the key projection heads according to the procedure described in \Cref{sec:init_saga}. We then reorder the teacher's query, value, and output projection heads using the induced permutations described in \Cref{app:permute_qvo} so that they remain aligned with the merged key heads, and copy all parameters unchanged from the teacher. The student is trained with a weighted combination of KL divergence from the teacher's output distribution (weight $\alpha = 0.7$, temperature $T = 3$) and hard-label cross-entropy (weight $1 - \alpha = 0.3$), using AdamW with fused implementation.

Training proceeds in three stages with increasing context length. In Stage~1, the student trains for 15K steps at context length 4096 with a learning rate of $2 \times 10^{-4}$ and global batch size 64. The learning rate schedule follows a warmup-stable-decay (WSD) pattern: 200 warmup steps, 9800 stable steps, and 5000 decay steps. The data mix is 75\% SlimPajama (packed), 15\% QA data (HotpotQA\footnote{\url{https://huggingface.co/datasets/hotpotqa/hotpot_qa}, CC BY-SA 4.0 license.} and SQuAD\footnote{\url{https://huggingface.co/datasets/rajpurkar/squad_v2}, CC BY-SA 4.0 license.}, packed), and 10\% synthetic needle-in-a-haystack examples.

Stage~2 starts from the Stage~1 checkpoint with a fresh optimizer and trains for 5K steps at context length 20480 with a learning rate of $1 \times 10^{-4}$ and global batch size 32. The WSD schedule uses 200 warmup, 3300 stable, and 1500 decay steps. The data mix shifts to 75\% SlimPajama, 5\% QA, 10\% long-context documents (LongAlign-10k\footnote{\url{https://huggingface.co/datasets/THUDM/LongAlign-10k}}, LongAlpaca-12k\footnote{\url{https://huggingface.co/datasets/Yukang/LongAlpaca-12k}}, individually padded), and 10\% synthetic NIAH.

Stage~3 continues from the Stage~2 checkpoint and trains at context length 36864 with a learning rate of $5 \times 10^{-5}$ and global batch size 32. For Qwen2.5-1.5B: 2K steps, $\alpha = 0.7$, $T = 3$, WSD schedule with 100 warmup and 500 decay steps. For Llama~3.2-1B: 4K steps, $\alpha = 0.85$, $T = 4$, WSD schedule with 200 warmup and 800 decay steps. Stage~3 interleaves 20K-length microbatches (2 per gradient step) with the 36K-length data to prevent forgetting of shorter-range attention patterns. Training at 36864 tokens provides headroom over the 32K evaluation target. Total training time is approximately 60 hours on 8$\times$H200 GPUs.

\subsection{Evaluation}
\label{app:evaluation}

Quality is evaluated using the RULER benchmark\footnote{\url{https://huggingface.co/datasets/THUDM/RULER}, Apache 2.0 license.} for long-context tasks, and HellaSwag\footnote{\url{https://huggingface.co/datasets/Rowan/hellaswag}, MIT license.} and WikiText\footnote{\url{https://huggingface.co/datasets/Salesforce/wikitext}, CC BY-SA 3.0 license.} for general capabilities. All evaluations use the lm-eval-harness framework\footnote{\url{https://github.com/EleutherAI/lm-evaluation-harness}, MIT license.}.

\section{Additional results}
\label{app:additional-results}

\subsection{Qwen2.5-1.5B speedups on H200}

\Cref{fig:speedup} reports Qwen2.5-1.5B latency measurements on B200. \cref{fig:speedup-qwen-h200} shows the corresponding results on H200, confirming that SAGA's speedup gains are consistent across GPU generations.

\begin{figure}[h]
  \centering
  \includegraphics[width=\textwidth]{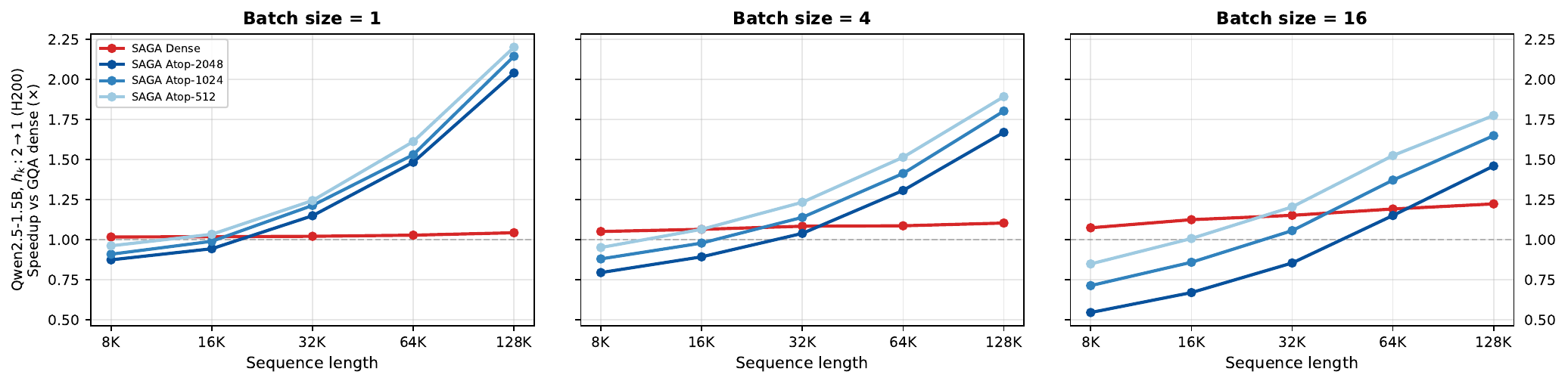}
  \caption{Speedup over baseline full attention on Qwen2.5-1.5B ($h_k: 2 \to 1$) measured on H200 at various batch sizes and sequence lengths.}
  \label{fig:speedup-qwen-h200}
\end{figure}

\subsection{Qwen3-0.6B trained from scratch on 50B tokens}
\label{app:additional-from-scratch}

We train a Qwen3-0.6B-based architecture from scratch on 50B tokens. This model is approximately $1.7\times$ larger than the SmolLM2-360M model evaluated in \cref{sec:scratch}, and its training budget is $9.54\times$ larger. We run a single seed for each configuration due to computational limitations. \cref{tab:qwen3-scratch-50b} shows that SAGA with $(h_k, h_v) = (1, 8)$ is closer to the high-capacity $(8, 8)$ baseline than to the matched-key-budget $(1, 1)$ baseline on HellaSwag accuracy, WikiText perplexity, and validation accuracy. Thus, the same trend holds at this larger model and training scale.

\begin{table}[h]
  \centering
  \caption{From-scratch results for a Qwen3-0.6B-based architecture trained on 50B tokens.}
  \label{tab:qwen3-scratch-50b}
  \small
  \begin{tabular}{lccc}
    \toprule
    Configuration & HellaSwag (acc\_norm $\uparrow$) & WikiText (word ppl $\downarrow$) & Validation accuracy ($\uparrow$) \\
    \midrule
    $(h_k, h_v) = (1, 1)$ & $0.3310 \pm 0.0047$ & 34.9129 & 0.440 \\
    $(h_k, h_v) = (1, 8)$ & $0.3343 \pm 0.0047$ & 33.8604 & 0.443 \\
    $(h_k, h_v) = (8, 8)$ & $0.3368 \pm 0.0047$ & 33.3382 & 0.444 \\
    \bottomrule
  \end{tabular}
\end{table}

\subsection{Llama~3.2-1B RULER accuracy}

\Cref{fig:ruler-accuracy-llama} presents the per-task RULER accuracy breakdown for Llama~3.2-1B ($h_k: 8 \to 4$), complementing the Qwen2.5-1.5B results in the main text. The SAGA student retains the majority of baseline quality, with degradation again concentrated in multi-key retrieval tasks.

\begin{figure}[h]
  \centering
  \includegraphics[width=\textwidth]{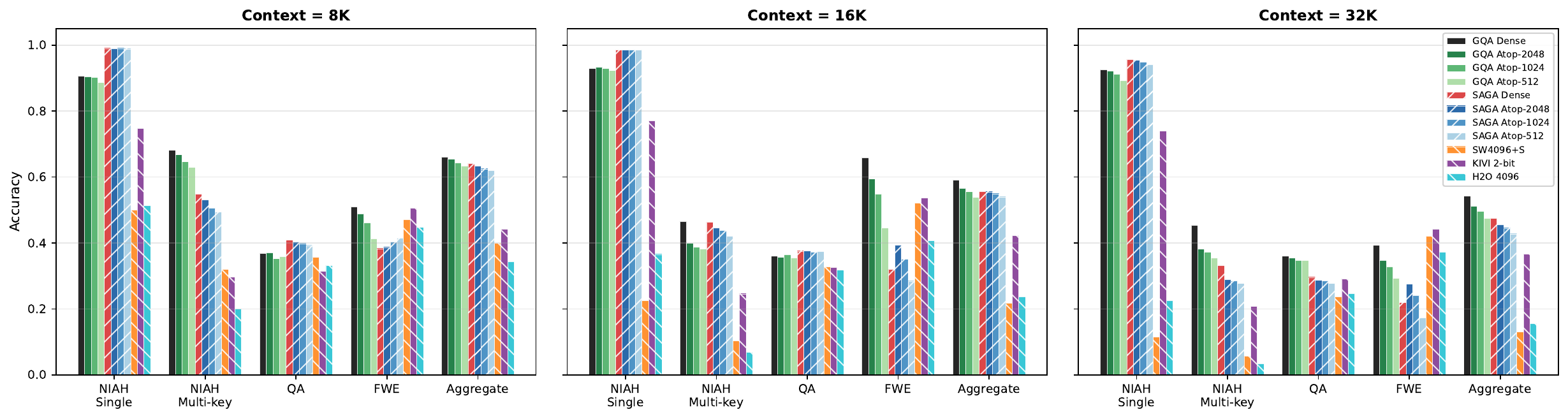}
  \caption{Per-task RULER accuracy for Llama~3.2-1B ($h_k: 8 \to 4$) with full attention and \atopn{} attention at $N \in \{512, 1024, 2048\}$, across context lengths 8K, 16K, and 32K.}
  \label{fig:ruler-accuracy-llama}
\end{figure}

\subsection{Kernel-level decode profile}
\label{app:profile}

\cref{tab:profile} profiles one decode step of Llama~3.2-1B at 64K context and batch size 16 on an H100. Under full attention the query-key and probability-value multiplications cost a comparable amount (870 vs.\ 1091\,$\mu$s). Merging key heads $8 \to 4$ halves the query-key term while leaving the probability-value term unchanged, and \atopn{} reduces the probability-value term to 88\,$\mu$s, leaving the query-key multiplication as the dominant remaining cost. The two components thus address the two bottlenecks of \cref{fig:saga-diagram}.

\begin{table}[h]
  \centering
  \caption{Decode-step profile of Llama~3.2-1B ($h_k: 8 \to 4$) at 64K context, batch size 16, on H100. Kernel times are $\mu$s per non-exempt layer per decode step; the last rows give end-to-end per-token latency.}
  \label{tab:profile}
  \small
  \begin{tabular}{lrrr}
    \toprule
    Kernel & GQA (full) & SAGA (full) & SAGA + \atopn{} 1024 \\
    \midrule
    Query-key multiplication & 870.1 & 466.0 & 446.0 \\
    Probability-value multiplication & 1090.9 & 1091.1 & 87.5 \\
    Softmax / threshold & 249.8 & 257.3 & 107.7 \\
    Gather & 1.2 & 1.2 & 88.8 \\
    Q/K/V/O projections & 22.8 & 23.3 & 22.8 \\
    MLP & 42.7 & 43.0 & 43.3 \\
    \midrule
    Layer total & 2277.6 & 1881.9 & 796.0 \\
    \midrule
    End-to-end (ms / token) & 40.57 & 31.08 & 18.99 \\
    Speedup & $1.00\times$ & $1.31\times$ & $2.14\times$ \\
    \bottomrule
  \end{tabular}
\end{table}

\subsection{Decode latency at larger model scale}
\label{app:scale}

Decode latency depends on tensor shapes and kernels rather than on weight values, so the efficiency of SAGA can be measured on larger architectures without training them. \cref{tab:scale-32b} reports per-token decode latency of the Qwen2.5-32B architecture ($h_k: 8 \to 4$) with untrained weights at 128K context on H100. The speedup of SAGA with \atopn{} persists at this scale and grows with batch size, as in the smaller models, though at a slower rate. Quality at this scale is not evaluated.

\begin{table}[h]
  \centering
  \caption{Per-token decode latency (ms) of the Qwen2.5-32B architecture ($h_k: 8 \to 4$) at 128K context on H100.}
  \label{tab:scale-32b}
  \small
  \begin{tabular}{lrrr}
    \toprule
    & Batch 1 & Batch 4 & Batch 16 \\
    \midrule
    GQA + full attention & 46.90 & 84.99 & 263.06 \\
    SAGA + \atopn{} 1024 & 39.43 & 68.93 & 175.06 \\
    Speedup & $1.19\times$ & $1.23\times$ & $1.50\times$ \\
    \bottomrule
  \end{tabular}
\end{table}

\subsection{Separating the head configuration from the selector}
\label{app:disentangle}

\cref{tab:disentangle} reports Qwen2.5-1.5B decode speedups for GQA and SAGA under full attention, \atopn{}, and \emph{exact} top-N, which recomputes the true top-$N$ indices with \texttt{torch.topk} at every step. Comparing GQA with and without \atopn{} isolates the selector; comparing SAGA against GQA under the same attention method isolates the head configuration. SAGA is faster under all three, e.g.\ $1.23\times$ over GQA with \atopn{} and $1.15\times$ with exact top-N at batch size 4 and 128K, so the architectural gain does not depend on the threshold approximation of \atopn{}. Exact top-N is slower than dense attention except at the longest contexts, which is what motivates \atopn{}.

\begin{table}[h]
  \centering
  \caption{Decode speedup of Qwen2.5-1.5B relative to GQA with full attention at the same batch size (H100), $N = 1024$.}
  \label{tab:disentangle}
  \small
  \begin{tabular}{cllrrrrr}
    \toprule
    Batch & Heads & Attention & 8K & 16K & 32K & 64K & 128K \\
    \midrule
    1 & GQA  & full        & 1.00 & 1.00 & 1.00 & 1.00 & 1.00 \\
    1 & GQA  & \atopn{}    & 0.93 & 1.00 & 1.17 & 1.43 & 1.88 \\
    1 & GQA  & exact top-N & 0.68 & 0.63 & 0.89 & 1.08 & 1.39 \\
    1 & SAGA & full        & 1.00 & 1.02 & 1.05 & 1.05 & 1.04 \\
    1 & SAGA & \atopn{}    & 0.95 & 1.02 & 1.23 & 1.53 & 2.06 \\
    1 & SAGA & exact top-N & 0.68 & 0.64 & 0.92 & 1.14 & 1.48 \\
    \midrule
    4 & GQA  & full        & 1.00 & 1.00 & 1.00 & 1.00 & 1.00 \\
    4 & GQA  & \atopn{}    & 0.85 & 0.93 & 1.06 & 1.24 & 1.51 \\
    4 & GQA  & exact top-N & 0.67 & 0.73 & 0.82 & 0.96 & 1.10 \\
    4 & SAGA & full        & 1.05 & 1.05 & 1.08 & 1.11 & 1.14 \\
    4 & SAGA & \atopn{}    & 0.89 & 0.98 & 1.16 & 1.41 & 1.86 \\
    4 & SAGA & exact top-N & 0.69 & 0.75 & 0.87 & 1.05 & 1.27 \\
    \bottomrule
  \end{tabular}
\end{table}


\ifshowchecklist
\newpage
\section*{NeurIPS Paper Checklist}

\begin{enumerate}

\item {\bf Claims}
    \item[] Question: Do the main claims made in the abstract and introduction accurately reflect the paper's contributions and scope?
    \item[] Answer: \answerYes{} 
    \item[] Justification: The paper claims, as described in the abstract, are reflected by the method and experiment sections.

\item {\bf Limitations}
    \item[] Question: Does the paper discuss the limitations of the work performed by the authors?
    \item[] Answer: \answerYes{} 
    \item[] Justification: Limitations are discussed in the limitation section.

\item {\bf Theory assumptions and proofs}
    \item[] Question: For each theoretical result, does the paper provide the full set of assumptions and a complete (and correct) proof?
    \item[] Answer: \answerYes{} 
    \item[] Justification: Theoretical results include complete proofs as well as assumptions.

    \item {\bf Experimental result reproducibility}
    \item[] Question: Does the paper fully disclose all the information needed to reproduce the main experimental results of the paper to the extent that it affects the main claims and/or conclusions of the paper (regardless of whether the code and data are provided or not)?
    \item[] Answer: \answerYes{} 
    \item[] Justification: Full implementation details are provided in the paper and supplementary material. Data used for evaluation and training is publicly available.

\item {\bf Open access to data and code}
    \item[] Question: Does the paper provide open access to the data and code, with sufficient instructions to faithfully reproduce the main experimental results, as described in supplemental material?
    \item[] Answer: \answerYes{}{} 
    \item[] Justification: Code implementation of our experiment is available in the supplementary material.

\item {\bf Experimental setting/details}
    \item[] Question: Does the paper specify all the training and test details (e.g., data splits, hyperparameters, how they were chosen, type of optimizer) necessary to understand the results?
    \item[] Answer: \answerYes{}{} 
    \item[] Justification: See implementation details.

\item {\bf Experiment statistical significance}
    \item[] Question: Does the paper report error bars suitably and correctly defined or other appropriate information about the statistical significance of the experiments?
    \item[] Answer: \answerNo{} 
    \item[] Justification: For the speedup experiments, error bars are not reported because the std is small enough to be ignored while making the graphs harder to follow. For the pretraining section, error bars are provided and defined.

\item {\bf Experiments compute resources}
    \item[] Question: For each experiment, does the paper provide sufficient information on the computer resources (type of compute workers, memory, time of execution) needed to reproduce the experiments?
    \item[] Answer: \answerYes{} 
    \item[] Justification: We provide these implementation details.
    
\item {\bf Code of ethics}
    \item[] Question: Does the research conducted in the paper conform, in every respect, with the NeurIPS Code of Ethics \url{https://neurips.cc/public/EthicsGuidelines}?
    \item[] Answer: \answerYes{}{} 
    \item[] Justification: This research conforms with the NeurIPS Code of Ethics.

\item {\bf Broader impacts}
    \item[] Question: Does the paper discuss both potential positive societal impacts and negative societal impacts of the work performed?
    \item[] Answer: \answerNA{} 
    \item[] Justification: The paper is on the topic of LLM architecture and efficiency. It has no societal impacts beyond the general in use of LLM technology.
    
\item {\bf Safeguards}
    \item[] Question: Does the paper describe safeguards that have been put in place for responsible release of data or models that have a high risk for misuse (e.g., pre-trained language models, image generators, or scraped datasets)?
    \item[] Answer: \answerNA{} 
    \item[] Justification: The paper poses no such risks.

\item {\bf Licenses for existing assets}
    \item[] Question: Are the creators or original owners of assets (e.g., code, data, models), used in the paper, properly credited and are the license and terms of use explicitly mentioned and properly respected?
    \item[] Answer: \answerYes{} 
    \item[] Justification: Datasets, models and code is properly cited and credited. Terms of used are properly respected.

\item {\bf New assets}
    \item[] Question: Are new assets introduced in the paper well documented and is the documentation provided alongside the assets?
    \item[] Answer: \answerYes{} 
    \item[] Justification: Documentation is provided.

\item {\bf Crowdsourcing and research with human subjects}
    \item[] Question: For crowdsourcing experiments and research with human subjects, does the paper include the full text of instructions given to participants and screenshots, if applicable, as well as details about compensation (if any)? 
    \item[] Answer: \answerNA{} 
    \item[] Justification: The paper does not involve crowdsourcing nor research with human subjects.

\item {\bf Institutional review board (IRB) approvals or equivalent for research with human subjects}
    \item[] Question: Does the paper describe potential risks incurred by study participants, whether such risks were disclosed to the subjects, and whether Institutional Review Board (IRB) approvals (or an equivalent approval/review based on the requirements of your country or institution) were obtained?
    \item[] Answer: \answerNA{} 
    \item[] Justification: The paper does not involve crowdsourcing nor research with human subjects.

\item {\bf Declaration of LLM usage}
    \item[] Question: Does the paper describe the usage of LLMs if it is an important, original, or non-standard component of the core methods in this research? Note that if the LLM is used only for writing, editing, or formatting purposes and does \emph{not} impact the core methodology, scientific rigor, or originality of the research, declaration is not required.
    \item[] Answer: \answerYes{} 
    \item[] Justification: The paper is on the topic of LLM architecture and efficency, and thus uses LLMs.

\end{enumerate}
\fi

\end{document}